\documentclass[11pt]{article}

\usepackage{natbib, authblk}
\usepackage[margin=1in]{geometry}

\usepackage[utf8]{inputenc} 
\usepackage[T1]{fontenc}    
\usepackage{hyperref}       
\usepackage{url}            
\usepackage{booktabs}       
\usepackage{amsfonts}       
\usepackage{nicefrac}       
\usepackage{microtype}      
\usepackage{xcolor}         

\usepackage{amsmath, enumitem, bbm, subcaption}
\usepackage{amsthm}
\usepackage{thmtools}
\usepackage{thm-restate}
\usepackage{tikz, svg}
\tikzset{separator/.style={dashed,white}}

\theoremstyle{plain}
\newtheorem{theorem}{Theorem}[section]
\newtheorem*{theorem*}{Theorem}
\newtheorem{corollary}[theorem]{Corollary}
\newtheorem{conjecture}[theorem]{Conjecture}
\theoremstyle{definition}
\newtheorem{definition}[theorem]{Definition}
\newtheorem*{definition*}{Definition}
\newtheorem*{example*}{Example}
\newtheorem*{remark*}{Remark}
\newtheorem{question}[theorem]{Question}

\usepackage{mathtools}
\usepackage{victor_preamble}
\newtheorem{claim}[theorem]{Claim}
\newtheorem{lemma}[theorem]{Lemma}

\title{Backdoor defense, learnability and obfuscation}

\author{Paul Christiano\thanks{Work done while at the Alignment Research Center prior to April 2024.}}
\author{Jacob Hilton\thanks{Corresponding author: \href{mailto:jacob@alignment.org}{\nolinkurl{jacob@alignment.org}}}}
\author{Victor Lecomte}
\author{Mark Xu\thanks{Authors ordered alphabetically.}}
\affil{Alignment Research Center}

\date{}

\begin{document}

\maketitle

\begin{abstract}
We introduce a formal notion of defendability against backdoors using a game between an attacker and a defender. In this game, the attacker modifies a function to behave differently on a particular input known as the ``trigger'', while behaving the same almost everywhere else. The defender then attempts to detect the trigger at evaluation time. If the defender succeeds with high enough probability, then the function class is said to be defendable. The key constraint on the attacker that makes defense possible is that the attacker's strategy must work for a randomly-chosen trigger.

Our definition is simple and does not explicitly mention learning, yet we demonstrate that it is closely connected to learnability. In the computationally unbounded setting, we use a voting algorithm of \citet{hanneke} to show that defendability is essentially determined by the VC dimension of the function class, in much the same way as PAC learnability. In the computationally bounded setting, we use a similar argument to show that efficient PAC learnability implies efficient defendability, but not conversely. On the other hand, we use indistinguishability obfuscation to show that the class of polynomial size circuits is not efficiently defendable. Finally, we present polynomial size decision trees as a natural example for which defense is strictly easier than learning. Thus, we identify efficient defendability as a notable intermediate concept in between efficient learnability and obfuscation.
\end{abstract}

\section{Introduction}

A \textit{backdoor} in a machine learning model is a modification to the model that causes it to behave differently on certain inputs that activate a secret ``trigger''. There is a wide literature on backdoor attacks and defenses from both a theoretical and an empirical perspective \citep{survey}. However, prior theoretical work typically makes reference to a particular training dataset, leading to a focus on data poisoning attacks. In this work we introduce a formal notion of a backdoor that allows the attacker to modify the model arbitrarily. Our definition is simple, but nevertheless gives rise to a rich array of strategies incorporating both learning and obfuscation.

Our formal notion is focused on backdoor detection \textit{at runtime}, meaning that the defender is given a particular input and must flag whether or not it activates the backdoor trigger. We focus on this case because of the existence of \textit{undetectable} backdoors if the defender is instead required to flag a model as backdoored without being given any particular input \citep*{goldwasser}. Moreover, detection at runtime is sufficient in threat models where the attacker is given only one opportunity to modify the model, being akin to giving the defender an additional chance to modify the model after the attacker. We also focus on the white-box setting, in which the defender has access to a complete description of the model.

\pagebreak[3]

Our main contributions, which are also summarized in Figures \ref{figure:definition} and \ref{figure:results}, are as follows:

\begin{figure}[t]
\makebox[\linewidth]{
\begin{tikzpicture}[box/.style={draw, rectangle, rounded corners=5pt, minimum height=5cm, text depth=5cm, align=left, anchor=north}]
\node[box, text width=3.5cm] at (0,0) {\vspace{1cm}\\\begin{enumerate}[leftmargin=*]\item Attacker chooses:\begin{itemize}[leftmargin=*]\item distribution $\mathcal D$\item original function\\$f\in\mathcal F$\end{itemize}\end{enumerate}};
\node[anchor=north] at (0,-0.1) {\includesvg[width=1cm, height=1cm]{images/swords.svg}};
\node[box, text width=3.6cm] at (4,0) {\vspace{1cm}\\\begin{enumerate}[leftmargin=*]\setcounter{enumi}{1}\item Chosen randomly:\begin{itemize}[leftmargin=*]\item backdoor trigger\\$x^\ast\sim\mathcal D$\end{itemize}\end{enumerate}};
\node[anchor=north] at (4,-0.1) {\includesvg[width=1cm, height=1cm]{images/questionmark.svg}};
\node[box, text width=4.4cm] at (8.45,0) {\vspace{1cm}\\\begin{enumerate}[leftmargin=*]\setcounter{enumi}{2}\item Attacker chooses:\begin{itemize}[leftmargin=*]\item backdoored function\\$f^\ast\in\mathcal F$ with\end{itemize}\end{enumerate}\vspace{-0.4cm}\begin{gather*}\mathbb P_{x\sim\mathcal D}\left(f^\ast\left(x\right)\neq f\left(x\right)\right)\leq\varepsilon\\\text{but}\quad f^\ast\left(x^\ast\right)\neq f\left(x^\ast\right)\end{gather*}};
\node[anchor=north] at (8.45,-0.1) {\includesvg[width=1cm, height=1cm]{images/swords.svg}};
\node[box, text width=2.75cm] at (12.475,0) {\vspace{1cm}\\\begin{enumerate}[leftmargin=*]\setcounter{enumi}{3}\item Defender distinguishes:\end{enumerate}\vspace{-0.2cm}\begin{gather*}\left(f,x\sim\mathcal D\right)\\\text{from}\\\left(f^\ast,x^\ast\right)\end{gather*}};
\node[anchor=north] at (12.475,-0.1) {\includesvg[width=1cm, height=1cm]{images/shield.svg}};
\end{tikzpicture}
}
\caption{The game used to define $\varepsilon$-defendability for a class $\mathcal F$ of $0,1$-valued functions.}
\label{figure:definition}
\end{figure}

\begin{itemize}
\item
\textbf{Definition of $\varepsilon$-defendability (Section~\ref{section:definition}).} We introduce a simple game between an attacker and a defender to define what it means for a representation class to be \textit{$\varepsilon$-defendable} with a certain confidence. The only constraint placed on the attacker is that their backdoor strategy must work for a randomly-chosen trigger. Without this assumption, there is a symmetry between the original and backdoored functions that makes defense impossible.
\item
\textbf{Statistical defendability (Section~\ref{section:unbounded}).} In the absence of computational constraints on the defender, we show that a representation class $\mathcal F$ is $\varepsilon$-defendable with confidence tending to $1$ if and only if $\varepsilon=o\left(\nicefrac 1{\operatorname{VC}\left(\mathcal F\right)}\right)$, where $\operatorname{VC}\left(\mathcal F\right)$ is the Vapnik--Chervonenkis dimension of $\mathcal F$. We achieve this by adapting a result of \citet{hanneke} that applies a learning algorithm multiple times and takes a majority vote.
\item
\textbf{Computational defendability (Section~\ref{section:bounded}).} We introduce a notion of \textit{efficient defendability} in which the defender's detection strategy must run in polynomial time. We show that efficient PAC learnability implies efficient defendability, but not conversely. We also show that under certain cryptographic assumptions, the class of polynomial size circuits is not efficiently defendable, by combining a puncturable pseudorandom function with an efficient indistinguishability obfuscator.
\item
\textbf{Defendability of decision trees (Section~\ref{section:trees}).} In the setting where the distribution over inputs is uniform, we give a defense for polynomial size decision trees that runs in the time taken to evaluate a single decision tree. This provides a natural example for which defense is faster than learning.
\end{itemize}

We conclude in Section~\ref{section:discussion} with a discussion of whether efficient defendability can be further separated from efficient PAC learnability, so-called ``mechanistic'' defenses that resemble our defense for polynomial size decision trees, and the implications of our results for AI alignment.

\begin{figure}[t]
\makebox[\linewidth]{
\begin{subfigure}[t]{0.33\textwidth}
\centering
\makebox[\linewidth]{
\begin{tikzpicture}
\draw[separator] (-7.55,-1.65) -- (-7.55,5.05);
\node[font=\bfseries,align=center,anchor=north] at (-5.15,5.75) {Statistical defendability\\(Section~\ref{section:unbounded})};
\node[align=center,font=\small] at (-5.15,1.9) {$\mathcal F$ is $\varepsilon$-defendable\\with confidence $\to 1$\\$\iff$\\$\varepsilon=o\left(\nicefrac 1{\operatorname{VC}\left(\mathcal F\right)}\right)$};
\node[align=center,font=\scriptsize] at (-5.15,0.9) {(by learning plus voting)};
\draw[separator] (-2.75,-1.65) -- (-2.75,5.05);
\end{tikzpicture}
}
\caption{Statistical defendability is concerned with computationally unbounded defenders, and is essentially determined by the VC dimension, in much the same way as statistical learnability.}
\end{subfigure}
\hspace{0.02\textwidth}
\begin{subfigure}[t]{0.35\textwidth}
\centering
\makebox[\linewidth]{
\begin{tikzpicture}
\draw[separator] (-2.75,-1.65) -- (-2.75,5.05);
\node[font=\bfseries,align=center,anchor=north] at (0,5.75) {Computational defendability\\(Section~\ref{section:bounded})};
\draw (0,0) ellipse (1.45cm and 1.25cm);
\node[align=center] at (0,0.5) {Efficiently\\PAC learnable};
\draw (0,0.9) ellipse (2.35cm and 2.35cm);
\node[align=center] at (0,2.6) {Efficiently\\defendable};
\node[circle,fill,inner sep=1.5pt] (random) at (0,1.75) {};
\node[anchor=west,align=center,font=\small] at (random.east) {$2^n$ random\\functions};
\node[circle,fill,inner sep=1.5pt] (circuits) at (0,4.1) {};
\node[anchor=west,align=center,font=\small] at (circuits.east) {Polynomial\\size circuits};
\node[anchor=west,align=center,font=\scriptsize,yshift=-20pt] at (circuits.east) {(by puncturing plus\\\hspace{0.75cm}obfuscation)};
\draw[separator] (2.75,-1.65) -- (2.75,5.05);
\end{tikzpicture}
}
\caption{Computational defendability is concerned with polynomial-time defenders, and the set of efficiently defendable classes lies strictly between the set of efficiently learnable classes and the set of all classes.}
\end{subfigure}
\hspace{0.02\textwidth}
\begin{subfigure}[t]{0.28\textwidth}
\centering
\makebox[\linewidth]{
\begin{tikzpicture}
\draw[separator] (2.75,-1.65) -- (2.75,5.05);
\node[font=\bfseries,align=center,anchor=north] at (4.95,5.75) {Defendability of\\decision trees\\(Section~\ref{section:trees})};
\node[align=center,font=\small] at (4.95,1.75) {Polynomial size\\decision trees\\are efficiently\\uniform-defendable\\in a single evaluation};
\draw[separator] (7.05,-1.65) -- (7.05,5.05);
\end{tikzpicture}
}
\caption{In the setting where the distribution over inputs is uniform, polynomial size decision trees are a natural class for which defense is faster than learning.}
\end{subfigure}
}
\caption{Summary of the results in this paper.}
\label{figure:results}
\end{figure}

\section{Related work}

The most closely related work to our own is that of \citet*{goldwasser}, whose main results use a digital signature scheme to insert an undetectable backdoor into a model. The backdoor trigger can be applied to any input by perturbing it appropriately. In the black-box setting, the backdoor is undetectable in the sense that it is computationally infeasible for the defender to find any input on which the backdoored model and the original model differ. In the white-box setting, the backdoor is undetectable in the sense that it is computationally infeasible for the defender to distinguish the backdoored model from the original model. However, both of these results are in the setting where the backdoor trigger must be extracted from the model rather than detected at runtime.

In the defense at runtime setting, the same authors also show that a backdoored model can be modified by the defender to produce a model that is not backdoored. Once again, the backdoor trigger is applied by perturbing the input. The result only applies to models that satisfy a certain smoothness assumption, and mirrors the use of randomized smoothing to achieve certified adversarial robustness by \citet{randomizedsmoothing}. By contrast, we avoid making smoothness assumptions, and further develop the theory of runtime backdoor detection.

Concurrent work by \citet*{oblivious} studies a black-box, non-realizable version of backdoor defense at runtime. Rather than predicting the exact label of the non-backdoored model on the given input (which is equivalent to our notion of runtime detection), the defender must instead produce a ``canonical'' label that has high accuracy on average, by using black-box access to the potentially backdoored model. In this local mitigation setting, they construct efficient defenders for almost-linear and almost-polynomial functions. They also consider the global mitigation setting, in which the defender must produce an entire clean model, constructing efficient global mitigators for Fourier-heavy functions.

Adversarial examples can be thought of as backdoor triggers, but the attacker must find similar inputs rather than a model with similar outputs. The gap between statistical and computational hardness has been explored in this context by \citet{advex1} and \citet{advex2}. Much like us, they find that defense is statistically possible but computationally hard.

Returning to backdoors, another thematically similar work to our own is that of \citet{rethinking}, which also questions the assumptions behind different backdoor attacks and defenses. However, it remains within the data-poisoning setting, in which the backdoor is inserted using corrupted training data.

A line of work beginning with \citet{semihandcrafted} goes beyond data-poisoning, inserting a backdoor into an already-trained model, but still using data to search for the backdoored model. More manual methods of backdoor insertion were later introduced by \citet{handcrafted}. This line of work is primarily empirical rather than theoretical.

Our result on statistical defendability is heavily based on the result of \citet{hanneke} in the realizable case. Following \citet{voting1} and \citet{voting2}, they use a voting algorithm to produce a backdoor defense in the data-poisoning setting. By combining this with a classical result of \citet*{hlw} on optimal learning, they relate the optimal performance of this algorithm to the VC dimension of the function class. Our result adapts this to our more general notion of defendability.

Our result on computational defendability of polynomial size circuits uses a hybrid argument that combines a puncturable pseudorandom function and an efficient indistinguishability obfuscator. This combination has been used in a number of previous hybrid arguments, as described by \citet{puncturableprfandio}. More generally, the relationship between learnability and cryptographic assumptions has a long history, with some basic constructions in this vein described in the textbook of \citet{computationallearningtheory}.

\section{Definition of \texorpdfstring{$\varepsilon$}{epsilon}-defendability}\label{section:definition}

Our formal notion of defendability against backdoors is based on a game played between two players, an ``attacker'' and a ``defender''. Informally, the attacker starts with a function and is given a randomly sampled backdoor trigger. They must choose a backdoored function that almost always agrees with the original function, but disagrees with the original function on the backdoor trigger. Then the defender is given either the original function and a random input, or the backdoored function and the backdoor trigger, and they must distinguish one from the other with high confidence.

For our formal definition, we fix a security parameter $n\in\mathbb N$ and consider classifiers $\mathcal X\to\left\{0,1\right\}$ where $\mathcal X=\left\{0,1\right\}^n$. Formally, a \textit{representation class} over $\mathcal X$ is a function $\left\{0,1\right\}^{R\left(n\right)}\to\left\{\text{all functions $\mathcal X\to\left\{0,1\right\}$}\right\}$ for some $R\left(n\right)\in\mathbb N$, but we generally keep the choice of representation implicit and conflate a representation class with the set of functions in its image. In our results on computationally constrained defenders, we also implicitly assume that $R\left(n\right)$ is polynomial in $n$.

Our definition also makes use of the \textit{example oracle} $\textsc{Ex}\left(f,\mathcal D\right)$ where $\mathcal D$ is a distribution over $\mathcal X$ and $f:\mathcal X\to\left\{0,1\right\}$. This is an oracle that outputs $\left(x,f\left(x\right)\right)$ for a randomly sampled $x\sim\mathcal D$.

\begin{definition}
Let $\mathcal F$ be a representation class over $\mathcal X=\left\{0,1\right\}^n$ and let $\varepsilon,\delta>0$. We say that $\mathcal F$ is \textit{$\varepsilon$-defendable with confidence $1-\delta$} if there is a probabilistic algorithm (the ``detection strategy'') that wins the following game with probability at least $1-\delta$:
\begin{itemize}
\item
An adversary with knowledge of the detection strategy (the ``attacker'') chooses a distribution $\mathcal D$ over $\mathcal X$ and $f\in\mathcal F$.
\item
We randomly sample $x^\ast\sim\mathcal D$ (the ``backdoor trigger''), and then the adversary chooses $f^\ast\in\mathcal F$ (the ``backdoored function'').
\item
The backdoored function $f^\ast$ is called \textit{$\varepsilon$-valid} if
\[\mathbb P_{x\sim\mathcal D}\left(f^\ast\left(x\right)\neq f\left(x\right)\right)\leq\varepsilon\quad\text{but}\quad f^\ast\left(x^\ast\right)\neq f\left(x^\ast\right).\]
\item
We randomly choose $\left(f^\prime,x^\prime\right)$ to pass to the detection strategy to be either $\left(f,x\right)$ for a random sample $x\sim\mathcal D$, or $\left(f^\ast,x^\ast\right)$, with 50\% probability each. The detection strategy is given access to the example oracle $\textsc{Ex}\left(f^\prime,\mathcal D\right)$ as well as the representation of $f^\prime$, $x^\prime$, $f^\prime\left(x^\prime\right)$, $\varepsilon$ and $\delta$ as input, and must output either $\textsc{Acc}$ (for ``accept'') or $\textsc{Rej}$ (for ``reject'').
\item
If $\left(f,x\right)$ is passed to the detection strategy, then the detection strategy wins if it outputs $\textsc{Acc}$. If $\left(f^\ast,x^\ast\right)$ is passed to the detection strategy, then the detection strategy wins if either it outputs $\textsc{Rej}$ or $f^\ast$ is not $\varepsilon$-valid.
\end{itemize}
\end{definition}

In this definition, the adversary is not specified by a collection of algorithms, but is instead shorthand for a universal quantifier. Thus we could have equivalently (but perhaps convolutedly) said, ``there exists a detection strategy such that, for all distributions $\mathcal D$ over $\mathcal X$ and all $f\in\mathcal F$, the expectation over $x^\ast\sim\mathcal D$ of the minimum over all $f^\ast\in\mathcal F$ of the win probability of the detection strategy is at least $1-\delta$''.

The manner in which $f$, $x^\ast$ and $f^\ast$ are chosen in this definition deserves some elaboration. If the attacker were allowed to choose the backdoor trigger $x^\ast$ themselves, then it would be too easy for the adversary to insert an undetectable backdoor. So we instead require the adversary to come up with a backdoor insertion strategy that works for a randomly-chosen trigger. In fact, we do not even allow the adversary to see the backdoor trigger when they are selecting the original function $f$, or else they could produce $f^\ast$ by \textit{removing an existing backdoor} rather than \textit{inserting a new one}. This is illustrated by the following example.

\begin{example*}[Existing backdoor removal]
Let $\mathcal X=\left\{0,1\right\}^n$ and let $\mathcal F=\left\{\mathbbm 1_x|\;x\in\mathcal X\right\}$, where $\mathbbm 1_x:\mathcal X\to\left\{0,1\right\}$ denotes the indicator function
\[\mathbbm 1_x\left(x^\prime\right)=\begin{cases}1,&\text{if $x^\prime=x$}\\0&\text{if $x^\prime\neq x$.}\end{cases}\]
Take $\mathcal D$ to be uniform over $\mathcal X$. If we allowed the adversary to choose both $f$ and $f^\ast$ after sampling $x^\ast\sim\mathcal D$, then they could take $f=\mathbbm 1_{x^\ast}$ and $f^\ast=\mathbbm 1_{x^\prime}$ for $x^\prime\sim\mathcal D$ (or in the event that $x^\prime=x^\ast$, instead take $f=\mathbbm 1_{x^{\prime\prime}}$ for some $x^{\prime\prime}\in\mathcal X\setminus\left\{x^\ast\right\}$ sampled uniformly at random). Then $f^\ast$ would be $\frac 2{2^n}$-valid, and $\left(f,x\right)$ and $\left(f^\ast,x^\ast\right)$ would be identically distributed for $x\sim\mathcal D$. So if $\varepsilon\geq\frac 1{2^{n-1}}$ then any detection strategy would win with probability at most $\frac 12$. This would be quite an extreme failure for the defender, especially considering that the VC dimension of $\mathcal F$ is only $1$.
\end{example*}

Thus the way in which $f$, $x^\ast$ and $f^\ast$ are chosen, by randomly sampling the backdoor trigger $x^\ast$ after the original function $f$ is chosen but before the backdoored function $f^\ast$ is chosen, is the key symmetry-breaking assumption that makes the game interesting.

\begin{remark*}
We could generalize this definition to include functions $\mathcal X\to\mathcal Y$ where $\mathcal Y$ is any subset of $\mathbb R$ (or even more generally, any metric space). In this setting, the most natural generalization of $\varepsilon$-validity is
\[\mathbb E_{x\sim\mathcal D}\left[\left(f^\ast\left(x\right)-f\left(x\right)\right)^2\right]\leq\varepsilon\quad\text{but}\quad \left|f^\ast\left(x^\ast\right)-f\left(x^\ast\right)\right|\geq 1.\]
\end{remark*}

\section{Statistical defendability}\label{section:unbounded}

In this section we will completely determine, up to a constant factor, the values of $\varepsilon$ and $\delta$ for which an arbitrary representation class is $\varepsilon$-defendable with confidence $1-\delta$, in the absence of any computational constraints on the detection strategy.

Our result is expressed in terms of the \textit{Vapnik--Chervonenkis (VC) dimension} of the representation class $\mathcal F$, denoted $\operatorname{VC}\left(\mathcal F\right)$. This is defined as the maximum size of a set $S$ of points \textit{shattered} by $\mathcal F$, meaning that all $2^{\left|S\right|}$ functions $S\to\left\{0,1\right\}$ are realized as the restriction $f\vert_S$ of some $f\in\mathcal F$.

\begin{restatable}{theorem}{restateunbounded}\label{theorem:unbounded}
Let $\mathcal F$ be a representation class over $\left\{0,1\right\}^n$ and let $\varepsilon>0$. Then the most (i.e., supremum) confidence with which $\mathcal F$ is $\varepsilon$-defendable is
\[\max\left(\tfrac 12,1-\Theta\left(\operatorname{VC}\left(\mathcal F\right)\varepsilon\right)\right)\]
as $\operatorname{VC}\left(\mathcal F\right)\to\infty$.
\end{restatable}

In particular, $\mathcal F$ is $\varepsilon$-defendable with confidence tending to $1$ as $n\to\infty$ if and only if $\varepsilon=o\left(\nicefrac 1{\operatorname{VC}\left(\mathcal F\right)}\right)$ as $n\to\infty$.

Thus in the computationally unbounded setting, defendability is determined almost entirely by the VC dimension of the representation class, in much the same way as PAC learnability \citep[Chapter 3]{behw,computationallearningtheory}.

In fact, to prove Theorem \ref{theorem:unbounded}, we will make use of a certain kind of learning algorithm called a \textit{prediction strategy}. Intuitively, a prediction strategy uses the value of a function on random samples to predict the value of the function on a new random sample. The formal definition is as follows.

\begin{definition*}
Let $\mathcal F$ be a representation class over $\mathcal X=\left\{0,1\right\}^n$. A \textit{prediction strategy} for $\mathcal F$ is a randomized algorithm that is given access to the example oracle $\textsc{Ex}\left(f,\mathcal D\right)$ for some distribution $\mathcal D$ over $\mathcal X$ and some $f\in\mathcal F$, as well as $x\in\mathcal X$ as input, and outputs a prediction for $f\left(x\right)$.

The \textit{sample size} of a prediction strategy is the maximum number of times it calls the example oracle, maximizing over any randomness used by the algorithm, any calls to the oracle, and the choice of $\mathcal D$, $f$ and $x$.

Given a choice of $\mathcal D$ and $f$, the \textit{error rate} of a prediction strategy is the probability that it fails to correctly predict $f\left(x\right)$ for $x\sim\mathcal D$, randomizing over any randomness used by the algorithm, any calls to the oracle, and the choice of $x$.
\end{definition*}

In their classical paper on prediction strategies, \citet*{hlw} exhibit a ``1-inclusion graph'' algorithm yielding the following result.

\begin{theorem*}[Haussler--Littlestone--Warmuth]
For any representation class $\mathcal F$ and any positive integer $m$, there is a prediction strategy for $\mathcal F$ with sample size $m-1$ such that for any distribution $\mathcal D$ over $\mathcal X$ and any $f\in\mathcal F$, the error rate is at most
\[\frac{\operatorname{VC}\left(\mathcal F\right)}m.\]
\end{theorem*}

As discussed in that work, there is a close relationship between prediction strategies and PAC learning algorithms, which we define in Section~\ref{section:learnable}. We focus on prediction strategies in this section because they make our proof easier to carry out while avoiding the introduction of unnecessary logarithmic factors.

Our proof of Theorem \ref{theorem:unbounded} is closely modeled on a proof of \citet[Theorem 3.1]{hanneke}, although we cannot quote that result directly since it is specialized to a data poisoning setup. As in that proof, our detection strategy works by applying the Haussler--Littlestone--Warmuth prediction strategy multiple times and taking a majority vote.

The basic idea for the detection strategy is that, given $\left(f^\prime,x^\prime\right)$ as input (which could be either $\left(f,x\right)$ or $\left(f^\ast,x^\ast\right)$), if we could successfully predict $f\left(x^\prime\right)$, then we could distinguish the two cases by comparing this to $f^\prime\left(x^\prime\right)$, since $f\left(x\right)=f\left(x\right)$ but $f\left(x^\ast\right)\neq f^\ast\left(x^\ast\right)$ if $f^\ast$ is $\varepsilon$-valid. To attempt to make this prediction, we use the Haussler--Littlestone--Warmuth prediction strategy. However, if $f^\prime=f^\ast$ then this can fail if we encounter a point where $f$ and $f^\ast$ disagree. Hence we are forced to use a small sample size, but we are still left with a small constant probability of failure. Fortunately though, we can repeat this procedure multiple times with independent samples and take a majority vote, which allows us to make this small constant probability vanish.

We now provide a sketch proof of Theorem \ref{theorem:unbounded}. For a full proof, see Appendix~\ref{appendix:unbounded}.

\begin{proof}[Sketch proof of Theorem \ref{theorem:unbounded}]
Write $d=\operatorname{VC}\left(\mathcal F\right)$, and note that the detection strategy can achieve a win probability of at least $\frac 12$ simply by guessing randomly.

To show that the most confidence with which $\mathcal F$ is $\varepsilon$-defendable is at most $\max\left(\tfrac 12,1-\Omega\left(d\varepsilon\right)\right)$, let $S$ be a set of size $d$ shattered by $\mathcal F$. Roughly speaking, the adversary can take $\mathcal D$ to be uniform over $S$, and take $f$ to be uniform over the witnesses to the shattering. Then the adversary can insert a backdoor by simply changing the value of $f$ at $x^\ast$ and no other points of $S$. The actual construction is slightly more careful than this, and is given in the full proof.

To show that the most confidence with which $\mathcal F$ is $\varepsilon$-defendable is at least $\max\left(\tfrac 12,1-O\left(d\varepsilon\right)\right)$, our detection strategy is as follows. Given $\left(f^\prime,x^\prime\right)$ as input, we apply the Haussler--Littlestone--Warmuth prediction strategy with sample size $m-1$ and error rate at most $\frac dm$ to make a prediction for $f^\prime\left(x^\prime\right)$, which we in turn think of as a prediction for $f\left(x^\prime\right)$. We repeat this procedure $r$ times with independent samples and take a majority vote: if more than half of the predictions are different from $f^\prime\left(x^\prime\right)$, then we output $\textsc{Rej}$, and otherwise we output $\textsc{Acc}$.

If $\left(f,x\right)$ is passed to the detection strategy, then the probability that more than $\frac 12$ of the votes are wrong is at most $\frac{2d}m$, by Markov's inequality. If $\left(f^\ast,x^\ast\right)$ is passed to the detection strategy, then there are two ways in which each vote can be wrong: either $f$ and $f^\ast$ disagree on at least one of the $m-1$ samples, or the prediction strategy fails to predict $f\left(x^\ast\right)$ even if it is provided with the value of $f$ on all of these $m-1$ samples. The probability that $\frac 14$ or more of the votes will be wrong due to the second failure mode is at most $\frac{4d}m$, by Markov's inequality again. Meanwhile, we can ensure that the probability of the first failure mode for each vote is at most $\frac 15$ by taking $\frac 1{5\varepsilon}<m\leq\frac 1{5\varepsilon}+1$. Since each choice of $m-1$ samples was independent, the probability that $\frac 14$ or more of the votes will be wrong due to the first failure mode is at most $\exp\left(-\frac r{200}\right)$, by Hoeffding's inequality, which can be made arbitrarily small by taking $r$ to be sufficiently large. Hence the probability that $\frac 12$ or more of the votes will be wrong is at most $\frac{4d}m$, and so the overall win probability is $1-\frac{3d}m=1-O\left(\varepsilon d\right)$, as required.
\end{proof}

Intuitively, the reason that this detection strategy works is that the if $\varepsilon$ is small compared to $\nicefrac 1{\operatorname{VC}\left(\mathcal F\right)}$, then the attacker must choose a backdoored function that is very ``strange''. Hence the detection strategy can sample a new function that is similar to the given possibly backdoored function, but more ``normal'', and see if the two functions agree on the given possible backdoor trigger. This process can be thought of as a kind of regularization.

While our proof used distillation (i.e., learning from function samples) plus ensembling (i.e., voting), other methods of regularization may also work. For example, we could also have sampled from a Boltzmann distribution centered on the given function (with exponentially decaying probability based on Hamming distance), to get the same result up to logarithmic factors. We prove this claim in Appendix~\ref{appendix:boltzmann}.

\section{Computational defendability}\label{section:bounded}

\subsection{Definition of efficient defendability}

In Section~\ref{section:unbounded}, we related the defendability of a representation class to its VC dimension. However, the prediction strategy we used to construct the detection strategy runs in exponential time, since it involves an expensive search over orientations of a certain graph. Hence it is interesting to ask~what a polynomial-time detection strategy can achieve. To explore this, we introduce the following notion.

\begin{definition}
Let $\mathcal F$ be a representation class over $\left\{0,1\right\}^n$. We say that $\mathcal F$ is \textit{efficiently defendable} if there is some polynomial $p$ such that for any $\delta>0$ and any $\varepsilon>0$ with $\varepsilon<\nicefrac 1{p\left(n,\frac 1\delta\right)}$, $\mathcal F$ is $\varepsilon$-defendable with confidence $1-\delta$ using a detection strategy that runs in time polynomial in $n$ and $\frac 1\delta$.
\end{definition}

Note that the condition that $\varepsilon<\nicefrac 1{p\left(n,\frac 1\delta\right)}$ is crucial: if $\mathcal F$ were required to be $\varepsilon$-defendable with confidence $1-\delta$ for any $\varepsilon,\delta>0$, then this would not be possible even for a detection strategy with unlimited time (except in trivial cases), by Theorem \ref{theorem:unbounded}. (Recall that $\varepsilon$ constrains the attacker, which is why this requirement helps the defender.)

As a reminder, we also assume in this section that representation classes have polynomial-length representations.

\subsection{Efficient defendability and efficient PAC learnability}\label{section:learnable}

In this sub-section we show that efficient PAC learnability implies efficient defendability, but not conversely.

The well-studied model of probably approximately correct (PAC) learning was introduced by \citet{paclearning}. The notion of efficient PAC learnability bears some resemblance to our notion of efficient defendability. The following definition is taken from \citet[Definition 4]{computationallearningtheory}.

\begin{definition*}
Let $\mathcal F$ be a representation class over $\mathcal X=\left\{0,1\right\}^n$. We say that $\mathcal F$ is \textit{efficiently PAC learnable} if there is a probabilistic algorithm (the ``PAC learning algorithm'') with the following properties. The PAC learning algorithm is given access to the example oracle $\textsc{Ex}\left(f,\mathcal D\right)$ for some distribution $\mathcal D$ over $\mathcal X$ and some $f\in\mathcal F$, as well as $0<\tilde\varepsilon<\frac 12$ (the ``error parameter'') and $0<\tilde\delta<\frac 12$ (the ``confidence parameter'') as input. The PAC learning algorithm must run in time polynomial in $n$, $\frac 1{\tilde\varepsilon}$ and $\frac 1{\tilde\delta}$, and must output some polynomially evaluatable hypothesis $h$, i.e. a representation for a function $\mathcal X\to\left\{0,1\right\}$ that is evaluatable in time polynomial in $n$. The hypothesis must satisfy $\mathbb P_{x\sim\mathcal D}\left(f\left(x\right)\neq h\left(x\right)\right)\leq\tilde\varepsilon$ with probability at least $1-\tilde\delta$ for any distribution $\mathcal D$ over $\mathcal X$ and any $f\in\mathcal F$.
\end{definition*}

Note that $\tilde\varepsilon$ and $\tilde\delta$ play subtly different roles in this definition to $\varepsilon$ and $\delta$ in the definition of efficient defendability. In particular, there is no condition on $\tilde\varepsilon$ as a function of $n$ and $\tilde\delta$ in the definition of efficient PAC learnability.

We now show that efficient PAC learnability implies efficient defendability. This follows easily from the proof of Theorem \ref{theorem:unbounded}. Even though that result used a prediction strategy that runs in exponential time, it is straightforward to replace it by an efficient PAC learning algorithm. 

\begin{corollary}\label{corollary:learnable}
Let $\mathcal F$ be a representation class over $\left\{0,1\right\}^n$. If $\mathcal F$ is efficiently PAC learnable, then $\mathcal F$ is efficiently defendable.
\end{corollary}

\begin{proof}
If $\mathcal F$ is efficiently PAC learnable, then for any sufficiently large positive integer $m$, there is a polynomial-time prediction strategy for $\mathcal F$ with sample size $m-1$ and error rate at most $\frac 16\delta$: we simply PAC learn a hypothesis with confidence parameter $\frac 1{12}\delta$ and error parameter $\frac 1{12}\delta$, and then evaluate this hypothesis on the given point. The definition of PAC learning guarantees that this succeeds as long as $m$ exceeds a particular polynomial in $n$ and $\frac 1\delta$.

Our detection strategy is then exactly the same as the one from the proof of Theorem \ref{theorem:unbounded}, but using the above prediction strategy instead of the Haussler--Littlestone--Warmuth prediction strategy. Since $\varepsilon<\nicefrac 1{p\left(n,\frac 1\delta\right)}$ for some polynomial $p$ of our choice, choosing $m>\frac 1{5\varepsilon}$ as in that proof suffices to ensure that $m$ is eventually large enough for the prediction strategy to succeed. Finally, we take the number of votes $r$ to be the smallest integer greater than $200\log\left(\frac 1\delta\right)$, which keeps the time complexity of the detection strategy polynomial in $n$ and $\frac 1\delta$. Then the detection strategy's overall win probability is at least $1-\frac 36\delta-\frac 12\exp\left(-\frac r{200}\right)>1-\delta$, as required.
\end{proof}

We now show that the converse implication does not hold in the random oracle model of computation. In this model, programs have access to an oracle that outputs the result of calling of a uniformly random function $\left\{0,1\right\}^\ast\to\left\{0,1\right\}$.

\begin{restatable}{theorem}{restaterandom}\label{theorem:random}
In the random oracle model of computation, with probability $1-O(2^{-n^c})$ for all $c$ over the choice of random oracle, there is a polynomially evaluatable representation class over $\mathcal X=\left\{0,1\right\}^n$ that is efficiently defendable but not efficiently PAC learnable.
\end{restatable}

To prove this, we simply take a representation class of $2^n$ functions that make unique calls to the random oracle (supplying a different input to the oracle whenever either the function is different or the input to the function is different), and show that, with high probability over the choice of random oracle, this representation class is efficiently defendable but not efficiently PAC learnable. Roughly speaking, it is efficiently defendable because, for most sets of $2^n$ random functions with $2^n$ possible inputs, every pair of functions differs on around half of inputs, meaning that there are no $\varepsilon$-valid functions for the adversary to choose from. Meanwhile, it is not efficiently PAC learnable because most random functions with $2^n$ possible inputs differ from every possible function that a polynomial-time algorithm could output on around half of inputs. For a full proof, see Appendix~\ref{appendix:random}.

Although this result is in the random oracle model of computation, the same proof shows that the converse does not hold in the usual model of computation either. This is because we can make the random choices ``on the outside'': we fix a randomly-chosen lookup table, and have the functions use that instead of calling the random oracle. However, this representation class is not necessarily polynomially evaluatable, since the lookup table would take up exponential space.

Nevertheless, we conjecture that there is a polynomially evaluatable counterexample that uses a pseudorandom function instead of a random oracle. The above proof cannot be immediately adapted to work for an arbitrary pseudorandom function, because two distinct keys could give rise to very similar functions, but it might not be too challenging to adapt it.

\begin{conjecture}
Assuming $\mathsf{OWF}$, there is a polynomially evaluatable representation class over $\left\{0,1\right\}^n$ that is efficiently defendable but not efficiently PAC learnable.
\end{conjecture}

Here, $\mathsf{OWF}$ denotes the existence of a one-way function, which can used to construct a pseudorandom function \citep[Chapter 6]{prfconstruction,cryptography}.

One way in which our counterexample is unsatisfying is that it relies on the adversary being unable to find a backdoored function that is $\varepsilon$-valid, making detection trivial. However, it is straightforward to extend this result to a class for which it is always possible for the adversary to find a backdoored function that is $\varepsilon$-valid, but that is nevertheless efficiently defendable.

\begin{example*}[Special-cased random functions]
Let $\widetilde{\mathcal F}$ be one of the above representation classes over $\mathcal X=\left\{0,1\right\}^n$ that serves as a counterexample (either $2^n$ functions that call a random oracle, or a random choice of $2^n$ functions). Now let
\[\mathcal F=\left\{x\mapsto\begin{cases}f\left(x\right),&\text{if $x\neq x^\ast$}\\y&\text{if $x=x^\ast$}\end{cases}\;\middle|\;f\in\widetilde{\mathcal F},x^\ast\in\mathcal X,y\in\left\{0,1\right\}\right\},\]
where implicitly, the representation of a function in $\mathcal F$ is given by a triple consisting of the representation of the function $f\in\widetilde{\mathcal F}$, the special-cased point $x^\ast\in\mathcal X$, and the special-cased value $y\in\left\{0,1\right\}$. Then the adversary can ensure that there is always a $\frac 2{2^n}$-valid backdoored function by taking $\mathcal D$ to be the uniform distribution over $\mathcal X$. However, $\mathcal F$ is still efficiently defendable with high probability, since the detection strategy can check whether or not the input has been special-cased. $\mathcal F$ also remains not efficiently PAC learnable with high probability.
\end{example*}

Thus we have an example of a non-trivial detection strategy that is faster than any PAC learning algorithm for that representation class. This shows that the learning-based detection strategy of Corollary \ref{corollary:learnable} is not always best possible under computational constraints. Nevertheless, our example is somewhat contrived, and so in Section~\ref{section:trees} we give an example of a more natural representation class for which defense is faster than learning.

\subsection{Efficient defendability and obfuscation}

We conclude our analysis of computational defendability with the following result, which essentially says that representation classes that are rich enough to support obfuscation are not efficiently defendable.

\begin{theorem}\label{theorem:bounded}
Assuming $\mathsf{OWF}$ and $\mathsf{iO}$, the representation class $\mathcal F$ of polynomial size Boolean circuits over $\mathcal X=\left\{0,1\right\}^n$ is not efficiently defendable.
\end{theorem}

Here, $\mathsf{OWF}$ again denotes the existence of a one-way function, and $\mathsf{iO}$ denotes the existence of an efficient indistinguishability obfuscator. Roughly speaking, an efficient indistinguishability obfuscator is a probabilistic polynomial-time algorithm that takes in a circuit and outputs an ``obfuscated'' circuit with the same behavior. The circuit being ``obfuscated'' means that the obfuscations of two different circuits with the same behavior are not distinguishable by any probabilistic polynomial-time adversary. For a precise definition of an efficient indistinguishability obfuscator, we refer the reader to \citet[Section 3]{puncturableprfandio}. The study indistinguishability obfuscation was initiated by \citet{io}, and an efficient indistinguishability obfuscator was constructed from well-studied computational hardness assumptions by \citet*{ioconstruction}.

We use $\mathsf{OWF}$ in our proof of Theorem \ref{theorem:bounded} to obtain a puncturable pseudorandom function. Roughly speaking, a puncturable pseudorandom function is pseudorandom function such that any key (i.e., seed) can be ``punctured'' at a set of points. The key being ``punctured'' means that, when run using the punctured key, the pseudorandom function behaves the same on unpunctured points, but to a probabilistic polynomial-time adversary with knowledge of the punctured key only, the function looks pseudorandom on punctured points when run using the original key. For a precise definition of a puncturable pseudorandom function, we again refer the reader to \citet[Section 3]{puncturableprfandio}. The observation that puncturable pseudorandom functions can be constructed from one-way functions was made by \citet{puncturableprf1}, \citet{puncturableprf2} and \citet{puncturableprf3}.

Our proof of Theorem \ref{theorem:bounded} is an example of a \textit{hybrid argument}: we show that two distributions are computationally indistinguishable via a sequence of intermediate distributions. For a precise definition of computational indistinguishability and further discussion of hybrid arguments, we refer the reader to \citet[Chapter 6.8]{cryptography}. The specific combination of a puncturable pseudorandom function and an efficient indistinguishability obfuscator has been used in a number of previous hybrid arguments, as described by \citet{puncturableprfandio}.

With these preliminaries in place, we are now ready to prove Theorem \ref{theorem:bounded}.

\begin{proof}[Proof of Theorem \ref{theorem:bounded}]
Take $\varepsilon=2^{-n}$. We will show that, for any polynomial $p$, the representation class $\mathcal F$ of polynomial size Boolean circuits is not $\varepsilon$-defendable with confidence $\frac 12+\frac 1{p\left(n\right)}$ using a detection strategy that runs in time polynomial in $n$.

To see this, by $\mathsf{OWF}$, let $C_K\in\mathcal F$ be the circuit for a puncturable pseudorandom function with key $K\in\left\{0,1\right\}^n$ and a single output bit, and by $\mathsf{iO}$, let $i\mathcal{O}$ be an efficient indistinguishability obfuscator. The adversary proceeds by taking $\mathcal D$ to be uniform over $\mathcal X$, and takes $f=i\mathcal O\left(C_K\right)$ for some $K\in\left\{0,1\right\}^n$ chosen uniformly at random. As before, we may treat $f$ as if it were chosen randomly, since the detection strategy's worst-case performance can be no better than its average-case performance. Finally, given $x^\ast$, the adversary takes
\[f^\ast=i\mathcal O\left(x\mapsto\begin{cases}C_{\left\langle\textsc{Punc}\left(K,x^\ast\right)\right\rangle}\left(x\right),&\text{if $x\neq x^\ast$}\\\left\langle 1-C_K\left(x^\ast\right)\right\rangle,&\text{if $x=x^\ast$}\end{cases}\right),\]
where $\textsc{Punc}\left(K,x^\ast\right)$ punctures the key $K$ at the point $x^\ast$, and angle brackets $\langle\dots\rangle$ indicate values that are hard-coded into the circuit, rather than being computed by the circuit. Note that $f^\ast$ is $\varepsilon$-valid since it agrees with $f$ on every point except $x^\ast$.

It remains to show that no polynomial-time detection strategy can win against this adversary with probability at least $\frac 12+\frac 1{p\left(n\right)}$, which is equivalent to saying that $\left(f,x\right)$ and $\left(f^\ast,x^\ast\right)$ are computationally indistinguishable for $x\sim\mathcal D$. To see this, consider the intermediate circuit
\[\tilde f=i\mathcal O\left(x\mapsto\begin{cases}C_{\left\langle\textsc{Punc}\left(K,x^\ast\right)\right\rangle}\left(x\right),&\text{if $x\neq x^\ast$}\\\left\langle C_K\left(x^\ast\right)\right\rangle,&\text{if $x=x^\ast$}\end{cases}\right),\]
which agrees with $f$ everywhere. We claim that
\[(f,x)\overset c\equiv(f,x^\ast)\overset c\equiv(\tilde f,x^\ast)\overset c\equiv(f^\ast,x^\ast),\]
where $\overset c\equiv$ denotes computational indistinguishability of distributions. The first equivalence is simply an identity of distributions, since $x^\ast$ was chosen randomly. The second equivalence follows by definition of indistinguishability obfuscation, with the technical detail that $x^\ast$ is used as the ``program state'' that is passed from the circuit generator to the distinguisher. The third equivalence follows from puncturability: if $(\tilde f,x^\ast)$ and $(f^\ast,x^\ast)$ were computationally distinguishable, then we could use this to computationally distinguish $C_K\left(x^\ast\right)$ from a uniformly random bit using only $x^\ast$ and $\textsc{Punc}\left(K,x^\ast\right)$. This final construction that relies on the fact that $i\mathcal O$ runs in polynomial time. By transitivity of computational indistinguishability, we have $\left(f,x\right)\overset c\equiv\left(f^\ast,x^\ast\right)$, as required.
\end{proof}

Thus even though all efficiently PAC learnable representation classes are efficiently defendable, some representation classes are not efficiently defendable due to the possibility of obfuscation.

\section{Defendability of decision trees}\label{section:trees}

In Section~\ref{section:learnable} we gave an example of a representation class for which defense is faster than learning, in the sense that it is efficiently defendable using a detection strategy that is faster than any PAC learning algorithm. However, our example was somewhat contrived, and was not polynomially evaluatable without access to a random oracle. In this section, we give an example of a more natural representation class for which defense is faster than learning: the class of polynomial size decision trees over $\left\{0,1\right\}^n$.

\begin{definition*}
A \textit{decision tree} over $\left\{0,1\right\}^n$ is a rooted binary tree where each non-leaf node is labeled with one of the $n$ input variables and each leaf node is labeled with a $0$ or a $1$. To evaluate a decision tree, we start at the root node and go left if the variable evaluates to $0$ and right if it evaluates to $1$, continuing until we reach a leaf. The \textit{size} of a decision tree is its number of leaves.
\end{definition*}

An example of a decision tree and how to evaluate it is given in Figure \ref{figure:trees}.

\begin{figure}[t]
\centering
\begin{tikzpicture}[
    level distance=1.2cm,
    every node/.style={circle,draw,minimum size=0.6cm,inner sep=0pt},
    level 1/.style={sibling distance=6cm},
    level 2/.style={sibling distance=3cm},
    level 3/.style={sibling distance=1.5cm},
    edge from parent/.style={draw,-latex}
]
    \node[red] (root) {$x_4$}
        child {
            node[red] (left) {$x_1$}
            child {
                node[red] (leftleft) {$x_2$}
                child {node[rectangle] {0}}
                child {node[rectangle, red] {1}}
            }
            child {node[rectangle] {1}}
        }
        child {
            node (right) {$x_2$}
            child {node[rectangle] {1}}
            child {node[rectangle] {0}}
        };
    \path[red, -latex] (root) edge node[left, draw=none, pos=0.3, xshift=-0.4cm, red] {0} (left);
    \path (root) edge node[right, draw=none, pos=0.3, xshift=0.4cm] {1} (right);
    \path[red, -latex] (left) edge node[left, draw=none, pos=0.3, xshift=-0.1cm, red] {0} (leftleft);
    \path (left) edge node[right, draw=none, pos=0.3, xshift=0.1cm] {1} (left-2);
    \path (leftleft) edge node[left, draw=none, pos=0.3] {0} (leftleft-1);
    \path[red, -latex] (leftleft) edge node[right, draw=none, pos=0.3] {1} (leftleft-2);
    \path (right) edge node[left, draw=none, pos=0.3, xshift=-0.1cm] {0} (right-1);
    \path (right) edge node[right, draw=none, pos=0.3, xshift=0.1cm] {1} (right-2);
\end{tikzpicture}
\caption{An example of a decision tree $f$ over $\left\{0,1\right\}^4$ with a red path showing how $f\left(0110\right)=1$.}
\label{figure:trees}
\end{figure}

In our study of decision trees, we focus on the \textit{uniform-PAC} model of learning, in which the distribution $\mathcal D$ is always the uniform distribution. Thus we say that a representation class is \textit{efficiently uniform-PAC learnable} to mean that it is efficiently PAC learnable as long as $\mathcal D$ is uniform. Similarly, we say that a representation class is \textit{efficiently uniform-defendable} to mean that it is efficiently defendable as long as $\mathcal D$ is uniform.

As far as we are aware, it is unknown whether the representation class of polynomial size decision trees over $\left\{0,1\right\}^n$ is efficiently uniform-PAC learnable, although polynomial-time PAC learning algorithms are known for certain product distributions \citep{decisiontrees}. However, the class is efficiently uniform PAC-learnable if we allow membership queries (i.e., calls to an oracle that outputs the value of the function on an input of the algorithm's choice) \citep[Chapter 3.5]{gl,booleanfunctions}. By specializing the proof of Corollary \ref{corollary:learnable}, it follows that the representation class of polynomial size decision trees over $\left\{0,1\right\}^n$ is efficiently defendable, since the defender can efficiently compute the results of membership queries for themselves using the representation of the decision tree.

Nevertheless, these learning algorithms for decision trees generally use at least linearly many calls to the example or membership oracle. By contrast, we now show that there is a defense for decision trees that is much faster than this.

\begin{restatable}{theorem}{restatetrees}\label{theorem:trees}
The representation class $\mathcal F$ of decision trees over $\mathcal X=\left\{0,1\right\}^n$ of size at most $s$, where $s$ is polynomial in $n$, is efficiently uniform-defendable using a detection strategy that makes 0 calls to the example oracle and runs in time $O\left(\text{time taken to evaluate a single decision tree}\right)$.
\end{restatable}

To prove this, we use a detection strategy that simply checks the depth of the leaf reached by the given input, and outputs $\textsc{Rej}$ if this is too large. Roughly speaking, the reason this works is that, in order for the backdoored function to be $\varepsilon$-valid, the depth of the leaf reached by the backdoor trigger must be large for either the original function or the backdoored function. But since our trees have polynomially many leaves, this depth is unlikely to be large for the original function. Hence if the depth of the given function on the given input is large, then it is more likely that we have been given the backdoored function and the backdoor trigger.

For a full proof of Theorem \ref{theorem:trees}, see Appendix~\ref{appendix:trees}. We suspect that a similar approach would lead to fast detection strategies for other natural representation classes of piecewise constant functions, perhaps even in the non-uniform setting.

The detection strategy in this proof is not directly comparable to a learning algorithm, since it does not use the example oracle at all. However, if we treat a call to the example oracle as having similar computational cost to evaluating a single decision tree, then it is substantially faster than any possible learning algorithm for this representation class. Thus there is more to defense than only learning.

\section{Discussion}\label{section:discussion}

\subsection{Separating efficient defendability from efficient PAC learnability}

A central theme of this work has been that learning can be used to perform backdoor defense, but backdoor defense cannot necessarily be used to perform learning. The first of these claims is encapsulated by Theorem \ref{theorem:unbounded} in the computationally unbounded setting and Corollary \ref{corollary:learnable} in the computationally bounded setting. For the second of these claims, we have made several steps in this direction:

\begin{itemize}
\item
We showed in Theorem \ref{theorem:random} that efficient defendability does not imply efficient PAC learnability in the random oracle model of computation.
\item
We deduced from this that efficient defendability does not imply efficient PAC learnability in the usual model of computation either, as long as we allow representation classes that are not polynomially evaluatable, and we conjectured that there is a polynomially evaluatable counterexample assuming the existence of a one-way function.
\item
We showed in Theorem \ref{theorem:trees} that the representation class of polynomial size decision trees is efficiently uniform-defendable using a detection strategy that is faster than any possible learning algorithm.
\end{itemize}

Nevertheless, we still do not have an example of a natural representation class that is efficiently defendable but not efficiently PAC learnable. We consider finding such an example to be a central challenge for follow-up research. One possible candidate of independent interest is the representation class of shallow (i.e., logarithmic-depth) polynomial size Boolean circuits, or equivalently, polynomial size Boolean formulas.

\begin{question}
Under reasonable computational hardness assumptions, is the representation class of polynomial size, logarithmic-depth Boolean circuits over $\left\{0,1\right\}^n$ efficiently defendable?
\end{question}

This representation class is known to not be efficiently PAC learnable under the assumption of the computational hardness of discrete cube roots \citep[Chapter 6]{circuitslearnability,computationallearningtheory}.

\subsection{Mechanistic defenses}\label{section:mechanistic}

Our detection strategy for decision trees in Theorem \ref{theorem:trees} is interesting not only because it is faster than any possible learning algorithm, but also because it works in a fundamentally different way. Given $\left(f^\prime,x^\prime\right)$ as input, it does not run the decision tree $f^\prime$ on any inputs other than $x^\prime$ itself, but instead exploits the \textit{mechanism} by which the value of $f^\prime\left(x^\prime\right)$ is computed, by looking at the depth of the corresponding leaf. We call a defense that exploits the structure of $f^\prime$ in this kind of way \textit{mechanistic}.

The existence of mechanistic defenses is another reason to suspect that there are other representation classes for which defense is strictly easier than learning. However, some sophistication may be required to construct such defenses. For example, one might hope to detect a backdoor trigger for a Boolean formula by checking the pattern of inputs to each gate, and seeing how unlikely that pattern would be for a random input to the formula. Unfortunately, though, the following example, found by Thomas Read, presents an obstacle to such an approach.

\begin{example*}[Backdoor with likely input patterns]
Consider the Boolean formulas $f,f^\ast:\left\{0,1\right\}^n\to\left\{0,1\right\}$ given by $f\left(x_1,\dots,x_n\right)=x_nx_{n-1}$ and
\[f^\ast\left(x_1,\dots,x_n\right)=\left(\dots\left(\left(x_1\vee x_nx_2\right)x_2\vee x_nx_3\right)x_3\vee\dots\vee x_nx_{n-1}\right)x_{n-1},\]
where $ab$ is shorthand for $\left(a\wedge b\right)$. By the distributive property, $f^\ast\left(x_1,\dots,x_n\right)$ is logically equivalent to $x_1x_2\dots x_{n-1}\vee x_nx_{n-1}$, and so $f$ and $f^\ast$ disagree only on the input $x^\ast:=\left(1,\dots,1,0\right)$. But on a uniformly random input to the formula, every gate in $f^\ast$ receives every pattern of inputs with probability at least $\nicefrac 18$. By contrast, in the logically equivalent version of $f^\ast$, the subformula $x_1x_2\dots x_{n-1}$ is $1$ on $x^\ast$, which only happens with probability $\nicefrac 1{2^{n-1}}$ on a uniformly random input.
\end{example*}

This example suggests a possible mechanistic defense for Boolean formulas that involves first rewriting the formula. More generally, mechanistic defenses offer an exciting avenue for future research.

\subsection{Implications for AI alignment}\label{section:alignment}

We are motivated to study backdoors as an analogy for \textit{deceptive alignment}, the possibility that an advanced AI system would learn to behave cooperatively when there are subtle cues that it is in training, but uncooperatively when those cues are missing \citep{rlo}. A deceptively aligned model is analogous to a backdoored function in that it behaves similarly to a fully cooperative model except on certain inputs that are rare during training. Thus, in this analogy, $\mathcal D$ is the training distribution, $f$ is the fully cooperative model, $f^\ast$ is the deceptively aligned model, and $x^\ast$ is an input on which $f^\ast$ behaves uncooperatively.

Note that allowing the backdoor to be detected at runtime is appropriate in this analogy, because we have the opportunity to modify $f^\ast$ or run a monitoring system. The flaw in the analogy comes from how the attacker is restricted. In our formal notion of defendability, the attacker must insert a backdoor that works for a randomly-chosen trigger. But in the case of deceptive alignment, there is no such restriction, and instead the backdoored behavior must arise from properties of the model architecture and the training distribution.

From this perspective, backdoor defenses that rely on learning are unsatisfying, because our assumption is that a similar process of learning gave rise to deceptive alignment in the first place. The detection strategy in our proof of Theorem \ref{theorem:unbounded} used distillation plus ensembling to produce a regularized model, but this may be exploiting the flaw in our analogy by using resampling to avoid the randomly-chosen trigger. Moreover, we would have to rely on a fast approximation to this strategy, since the exact version runs in exponential time. Similar issues apply to other methods of regularization, such as the one discussed in Appendix~\ref{appendix:boltzmann}, although that alternative does show more promise.

On the other hand, mechanistic defenses, as discussed in Section~\ref{section:mechanistic}, may fare better, since they work very differently. Intuitively, if a detection strategy could spot the mechanism by which a deceptively aligned model concluded that it was in training, then it should transfer well to the case of deceptive alignment. A mechanistic defense may be able to do this by exploiting the fact that this mechanism is active on the backdoor trigger but inactive on most random inputs. Unfortunately though, we are lacking in examples of mechanistic defenses, which why we are excited to see more research in this direction.

Our result that polynomial size circuits are not efficiently defendable in Theorem \ref{theorem:bounded} is also relevant to this analogy. Although it is unlikely that our indistinguishability obfuscator-based construction would arise out of ordinary model training, it is much more plausible for a trained model to be obfuscated in a more informal sense. Indeed, reverse engineering trained neural networks is an active area of research \citep{mechinterp}. Hence the possibility of obfuscation poses a potential problem for detecting deceptive alignment. However, in the case of deceptive alignment, we also have access to the entire training process, including the training dataset, which may be enough information for us to detect the trigger despite any potential obfuscation. By analogy, in our construction in Theorem \ref{theorem:bounded}, it would be easy for the defender to detect the trigger if they had access to the unpunctured key $K$ that was used in the construction of $f^\ast$.

This motivates the study of variants of our formal notion of defendability that constrain the attacker in different ways, or provide more assistance to the defender. To better capture the analogy with deceptive alignment, we would be excited to see research into variants that provide the defender with more information about how the function they are given was constructed, to see if this makes defense computationally feasible.

\section{Conclusion}

We have introduced a formal notion of defendability against backdoors in which the attacker's strategy must work for a randomly-chosen trigger. Despite its simplicity, this notion gives rise to a rich array of strategies. In the absence of computational constraints, defense is exactly as hard as learning. Meanwhile, in the presence of computational constraints, defense is strictly easier than learning, but impossible for function classes that are rich enough to support obfuscation. We are excited to see future work that further explores the exact relationship between defense and learning.

\section{Acknowledgments}

We are grateful to Dmitry Vaintrob for an earlier version of the results in Appendix~\ref{appendix:boltzmann}; to Thomas Read for finding the ``Backdoor with likely input patterns'' example and for help with proofs; to Andrea Lincoln, D\'{a}vid Matolcsi, Eric Neyman, George Robinson and Jack Smith for contributions to the project in its early stages; and to Geoffrey Irving, Robert Lasenby and Eric Neyman for helpful comments on drafts.

\bibliographystyle{abbrvnat}
\bibliography{bibliography}

\begin{thebibliography}{33}
\providecommand{\natexlab}[1]{#1}
\providecommand{\url}[1]{\texttt{#1}}
\expandafter\ifx\csname urlstyle\endcsname\relax
  \providecommand{\doi}[1]{doi: #1}\else
  \providecommand{\doi}{doi: \begingroup \urlstyle{rm}\Url}\fi

\bibitem[Barak et~al.(2001)Barak, Goldreich, Impagliazzo, Rudich, Sahai,
  Vadhan, and Yang]{io}
B.~Barak, O.~Goldreich, R.~Impagliazzo, S.~Rudich, A.~Sahai, S.~Vadhan, and
  K.~Yang.
\newblock On the (im)possibility of obfuscating programs.
\newblock In \emph{Annual international cryptology conference}, pages 1--18.
  Springer, 2001.

\bibitem[Blumer et~al.(1989)Blumer, Ehrenfeucht, Haussler, and Warmuth]{behw}
A.~Blumer, A.~Ehrenfeucht, D.~Haussler, and M.~K. Warmuth.
\newblock Learnability and the {V}apnik-{C}hervonenkis dimension.
\newblock \emph{Journal of the ACM (JACM)}, 36\penalty0 (4):\penalty0 929--965,
  1989.

\bibitem[Boneh and Waters(2013)]{puncturableprf1}
D.~Boneh and B.~Waters.
\newblock Constrained pseudorandom functions and their applications.
\newblock In \emph{Advances in Cryptology-ASIACRYPT 2013: 19th International
  Conference on the Theory and Application of Cryptology and Information
  Security, Bengaluru, India, December 1-5, 2013, Proceedings, Part II 19},
  pages 280--300. Springer, 2013.

\bibitem[Boyle et~al.(2014)Boyle, Goldwasser, and Ivan]{puncturableprf2}
E.~Boyle, S.~Goldwasser, and I.~Ivan.
\newblock Functional signatures and pseudorandom functions.
\newblock In \emph{International workshop on public key cryptography}, pages
  501--519. Springer, 2014.

\bibitem[Bubeck et~al.(2019)Bubeck, Lee, Price, and Razenshteyn]{advex1}
S.~Bubeck, Y.~T. Lee, E.~Price, and I.~Razenshteyn.
\newblock Adversarial examples from computational constraints.
\newblock In \emph{International Conference on Machine Learning}, pages
  831--840. PMLR, 2019.

\bibitem[Cohen et~al.(2019)Cohen, Rosenfeld, and Kolter]{randomizedsmoothing}
J.~Cohen, E.~Rosenfeld, and Z.~Kolter.
\newblock Certified adversarial robustness via randomized smoothing.
\newblock In \emph{international conference on machine learning}, pages
  1310--1320. PMLR, 2019.

\bibitem[Dumford and Scheirer(2020)]{semihandcrafted}
J.~Dumford and W.~Scheirer.
\newblock Backdooring convolutional neural networks via targeted weight
  perturbations.
\newblock In \emph{2020 IEEE International Joint Conference on Biometrics
  (IJCB)}, pages 1--9. IEEE, 2020.

\bibitem[Dziugaite and Roy(2017)]{dziugaite2017computing}
G.~K. Dziugaite and D.~M. Roy.
\newblock Computing nonvacuous generalization bounds for deep (stochastic)
  neural networks with many more parameters than training data.
\newblock \emph{arXiv preprint arXiv:1703.11008}, 2017.

\bibitem[Garg et~al.(2020)Garg, Jha, Mahloujifar, and Mohammad]{advex2}
S.~Garg, S.~Jha, S.~Mahloujifar, and M.~Mohammad.
\newblock Adversarially robust learning could leverage computational hardness.
\newblock In \emph{Algorithmic Learning Theory}, pages 364--385. PMLR, 2020.

\bibitem[Gentile and Helmbold(1998)]{gentile1998improved}
C.~Gentile and D.~P. Helmbold.
\newblock Improved lower bounds for learning from noisy examples: An
  information-theoretic approach.
\newblock In \emph{Proceedings of the eleventh annual conference on
  Computational learning theory}, pages 104--115, 1998.

\bibitem[Goldreich and Levin(1989)]{gl}
O.~Goldreich and L.~A. Levin.
\newblock A hard-core predicate for all one-way functions.
\newblock In \emph{Proceedings of the twenty-first annual ACM symposium on
  Theory of computing}, pages 25--32, 1989.

\bibitem[Goldreich et~al.(1986)Goldreich, Goldwasser, and
  Micali]{prfconstruction}
O.~Goldreich, S.~Goldwasser, and S.~Micali.
\newblock How to construct random functions.
\newblock \emph{Journal of the ACM (JACM)}, 33\penalty0 (4):\penalty0 792--807,
  1986.

\bibitem[Goldwasser et~al.(2022)Goldwasser, Kim, Vaikuntanathan, and
  Zamir]{goldwasser}
S.~Goldwasser, M.~P. Kim, V.~Vaikuntanathan, and O.~Zamir.
\newblock Planting undetectable backdoors in machine learning models.
\newblock In \emph{2022 IEEE 63rd Annual Symposium on Foundations of Computer
  Science (FOCS)}, pages 931--942. IEEE, 2022.

\bibitem[Goldwasser et~al.(2024)Goldwasser, Shafer, Vafa, and
  Vaikuntanathan]{oblivious}
S.~Goldwasser, J.~Shafer, N.~Vafa, and V.~Vaikuntanathan.
\newblock Oblivious defense in {ML} models: Backdoor removal without detection.
\newblock \emph{arXiv preprint arXiv:2411.03279}, 2024.

\bibitem[Hanneke et~al.(2022)Hanneke, Karbasi, Mahmoody, Mehalel, and
  Moran]{hanneke}
S.~Hanneke, A.~Karbasi, M.~Mahmoody, I.~Mehalel, and S.~Moran.
\newblock On optimal learning under targeted data poisoning.
\newblock \emph{Advances in Neural Information Processing Systems},
  35:\penalty0 30770--30782, 2022.

\bibitem[Haussler et~al.(1994{\natexlab{a}})Haussler, Kearns, and
  Schapire]{haussler1994bounds}
D.~Haussler, M.~Kearns, and R.~E. Schapire.
\newblock Bounds on the sample complexity of bayesian learning using
  information theory and the vc dimension.
\newblock \emph{Machine learning}, 14:\penalty0 83--113, 1994{\natexlab{a}}.

\bibitem[Haussler et~al.(1994{\natexlab{b}})Haussler, Littlestone, and
  Warmuth]{hlw}
D.~Haussler, N.~Littlestone, and M.~K. Warmuth.
\newblock Predicting $\{$0, 1$\}$-functions on randomly drawn points.
\newblock \emph{Information and Computation}, 115\penalty0 (2):\penalty0
  248--292, 1994{\natexlab{b}}.

\bibitem[Hong et~al.(2022)Hong, Carlini, and Kurakin]{handcrafted}
S.~Hong, N.~Carlini, and A.~Kurakin.
\newblock Handcrafted backdoors in deep neural networks.
\newblock \emph{Advances in Neural Information Processing Systems},
  35:\penalty0 8068--8080, 2022.

\bibitem[Hubinger et~al.(2019)Hubinger, van Merwijk, Mikulik, Skalse, and
  Garrabrant]{rlo}
E.~Hubinger, C.~van Merwijk, V.~Mikulik, J.~Skalse, and S.~Garrabrant.
\newblock Risks from learned optimization in advanced machine learning systems.
\newblock \emph{arXiv preprint arXiv:1906.01820}, 2019.

\bibitem[Jain et~al.(2021)Jain, Lin, and Sahai]{ioconstruction}
A.~Jain, H.~Lin, and A.~Sahai.
\newblock Indistinguishability obfuscation from well-founded assumptions.
\newblock In \emph{Proceedings of the 53rd Annual ACM SIGACT Symposium on
  Theory of Computing}, pages 60--73, 2021.

\bibitem[Jia et~al.(2021)Jia, Cao, and Gong]{voting2}
J.~Jia, X.~Cao, and N.~Z. Gong.
\newblock Intrinsic certified robustness of bagging against data poisoning
  attacks.
\newblock In \emph{Proceedings of the AAAI conference on artificial
  intelligence}, volume~35, pages 7961--7969, 2021.

\bibitem[Kalai and Teng(2008)]{decisiontrees}
A.~T. Kalai and S.-H. Teng.
\newblock Decision trees are {PAC}-learnable from most product distributions: a
  smoothed analysis.
\newblock \emph{arXiv preprint arXiv:0812.0933}, 2008.

\bibitem[Katz and Lindell(2007)]{cryptography}
J.~Katz and Y.~Lindell.
\newblock \emph{Introduction to modern cryptography: principles and protocols}.
\newblock Chapman and hall/CRC, 2007.

\bibitem[Kearns and Valiant(1994)]{circuitslearnability}
M.~Kearns and L.~Valiant.
\newblock Cryptographic limitations on learning {B}oolean formulae and finite
  automata.
\newblock \emph{Journal of the ACM (JACM)}, 41\penalty0 (1):\penalty0 67--95,
  1994.

\bibitem[Kearns and Vazirani(1994)]{computationallearningtheory}
M.~J. Kearns and U.~Vazirani.
\newblock \emph{An introduction to computational learning theory}.
\newblock MIT press, 1994.

\bibitem[Khaddaj et~al.(2023)Khaddaj, Leclerc, Makelov, Georgiev, Salman,
  Ilyas, and Madry]{rethinking}
A.~Khaddaj, G.~Leclerc, A.~Makelov, K.~Georgiev, H.~Salman, A.~Ilyas, and
  A.~Madry.
\newblock Rethinking backdoor attacks.
\newblock In \emph{International Conference on Machine Learning}, pages
  16216--16236. PMLR, 2023.

\bibitem[Kiayias et~al.(2013)Kiayias, Papadopoulos, Triandopoulos, and
  Zacharias]{puncturableprf3}
A.~Kiayias, S.~Papadopoulos, N.~Triandopoulos, and T.~Zacharias.
\newblock Delegatable pseudorandom functions and applications.
\newblock In \emph{Proceedings of the 2013 ACM SIGSAC conference on Computer \&
  communications security}, pages 669--684, 2013.

\bibitem[Levine and Feizi(2020)]{voting1}
A.~Levine and S.~Feizi.
\newblock Deep partition aggregation: Provable defense against general
  poisoning attacks.
\newblock \emph{arXiv preprint arXiv:2006.14768}, 2020.

\bibitem[Li et~al.(2022)Li, Jiang, Li, and Xia]{survey}
Y.~Li, Y.~Jiang, Z.~Li, and S.-T. Xia.
\newblock Backdoor learning: A survey.
\newblock \emph{IEEE Transactions on Neural Networks and Learning Systems},
  2022.

\bibitem[O'Donnell(2021)]{booleanfunctions}
R.~O'Donnell.
\newblock Analysis of {B}oolean functions.
\newblock \emph{arXiv preprint arXiv:2105.10386}, 2021.

\bibitem[Olah(2022)]{mechinterp}
C.~Olah.
\newblock Mechanistic interpretability, variables, and the importance of
  interpretable bases.
\newblock \emph{Transformer Circuits Thread}, 2022.
\newblock URL
  \url{https://www.transformer-circuits.pub/2022/mech-interp-essay}.

\bibitem[Sahai and Waters(2014)]{puncturableprfandio}
A.~Sahai and B.~Waters.
\newblock How to use indistinguishability obfuscation: deniable encryption, and
  more.
\newblock In \emph{Proceedings of the forty-sixth annual ACM symposium on
  Theory of computing}, pages 475--484, 2014.

\bibitem[Valiant(1984)]{paclearning}
L.~G. Valiant.
\newblock A theory of the learnable.
\newblock \emph{Communications of the ACM}, 27\penalty0 (11):\penalty0
  1134--1142, 1984.

\end{thebibliography}

\newpage

\appendix

\section{Alternative version of Theorem \ref{theorem:unbounded} using a Boltzmann posterior}\label{appendix:boltzmann}

In this section we prove a slightly weaker version of Theorem~\ref{theorem:unbounded} using a detection strategy that does not involve learning. Instead of using distillation plus ensembling as in the proof of Theorem~\ref{theorem:unbounded}, we employ an alternative method of regularization that involves resampling the given function from a Boltzmann posterior. This gives rise to an extra logarithmic factor in the confidence of the detection algorithm, but otherwise gives the same bound, and in particular the threshold for detecting backdoors in a class of VC dimension $d$ is still $\eps = \Th\p{\f1d}$.

We are grateful to Dmitry Vaintrob for an earlier version of these results.

\subsection{General strategy}
We first present the strategy generally. This strategy also aims to guess the value $f(\Bx^*)$ based on $\Bf'$. But instead of doing this by creating a new dataset from $\Bf'$ then discarding $\Bf'$ completely, it does this by averaging over a Boltzmann posterior centered around the function $\Bf'$.

At an intuitive level, the defender doesn't trust the function $\Bf'$ that it received on every point, but they do know that it's close to $f$ overall. So it makes sense for them to ``squint'' by looking through the neighborhood of $\Bf'$ within the space $\CF$ of possible functions, and average the outputs that they give for $\Bx^*$. Indeed, if we average over a ball centered at $\Bf'$ that is both
\begin{itemize}
\item ``accurate'': small enough that the functions considered are close to $\Bf'$ and therefore $f$,
\item ``secure'': large enough that there is not much difference between the ball centered at $f$ and the ball around centered at $\Bf^*$,
\end{itemize}

then the answers we'll get will be overall accurate \emph{without} being very sensitive to $\Bx^*$.

Let's start by defining the distribution over which we will take an average. First, let $P \in \tri\p\CF$ be any prior distribution over hypotheses. The performance of detection strategy will depend on which prior $P$ is chosen, but for now we'll keep it unfixed. 

We can now define the Boltzmann posterior.

\begin{definition}\label{def:b1}
    For two functions $f,f' \in \CF$, let $\abs{f-f'} \ce \Pr_{\Bx \sim \CD}\b{f(\Bx) \ne f'(\Bx)}$ denote the distance between $f$ and $f'$ under $\CD$.
\end{definition}

\begin{definition}\label{def:b2}
    Fix some integer $m \ge 1$. Let $Q_f= Q_f(P,m) \in \tri(\CF)$ be the distribution given by $Q_f(h) \propto P(h)\exp\p{-m\abs{h-f}}$.
\end{definition}

Here, we can think of $m$ as an ``inverse temperature'': the larger $m$ is, the more concentrated $Q_f$ will be around $f$. It can also be understood as the ``number of samples we're drawing'' to compare $h$ and $f$, since
$\Pr_{\Bx \sim \CD^m}\b{\forall i:h(\Bx_i) = f(\Bx_i)} = \p{1-\abs{h-f}}^m \approx \exp\p{-m\abs{h-f}}$.
\begin{definition}\label{def:b3}
    Let the \emph{Boltzmann strategy} be the detection strategy which draws some hypothesis $\Bh$ at random from $Q_{\Bf'}$ then returns $\nc{\acc}{\mathtt{Acc}}\acc$ if $\Bh(\Bx') = \Bf'(\Bx')$, and $\nc{\rej}{\texttt{Rej}}\rej$ otherwise.
\end{definition}

\begin{claim}\label{claim:b4}
    The Boltzmann strategy fails on $f$ with probability $\Pr\ss{\p{\Bx',\Bf'}\\\Bh \sim Q_{\Bf'}}\b{\Bh\p{\Bx'} \ne f\p{\Bx'}}$.
\end{claim}

\begin{proof}
    The game guarantees that $f\p{\Bx'} = \Bf'\p{\Bx'}$ iff $\Bf'=f$, so the Boltzmann strategy will give the right answer as long as $\Bh\p{\Bx'} = f\p{\Bx'}$.
\end{proof}

We can already show that $Q_f$ is ``secure'': if two functions $f,f'$ are close to each other, then the posteriors $Q_f$ and $Q_{f'}$ must also be similar, since the likelihoods will not differ by much.

\begin{claim}\label{claim:b5}
    $\forall f,f',h \in \CF: Q_{f'}(h) \le \exp\p{2m\abs{f-f'}} Q_f(h)$.
\end{claim}

\begin{proof}
    Let $C_f \ce \E_{\Bh \sim P}\b{\exp\p{-m\abs{\Bh-f}}}$ and $C_{f'} \ce \E_{\Bh \sim P}\b{\exp\p{-m\abs{\Bh-f'}}}$. By the triangle inequality, for any $h \in \CF$ we have
$\abs{\strut\abs{h-f} - \abs{h-f'}} \le \abs{f-f'}$,
so $C_{f'} \ge \exp(-m\abs{f-f'}) C_f$ and
\begin{align*}
\f{Q_{f'}(h)}{Q_{f}(h)}
&= \f{P(h)\exp\p{-m\abs{h-f'}}/C_{f'}}{P(h)\exp\p{-m\abs{h-f}}/C_{f}}\\
&= \exp\p{m\p{\abs{h-f}-\abs{h-f'}}}\f{C_f}{C_f'}\\
&\le \exp\p{2m\abs{f-f'}}.\qedhere
\end{align*}
\end{proof}

In particular, when $\abs{f-f'} \le \eps$, this means that if on some input $x'$, $Q_f$ disagrees with $f$ with probability $\le\delta$, then $Q_{f'}$ disagrees with $f$ with probability $\le\exp\p{2m\eps}\delta$. And crucially, this is true even if $x'$ is chosen adversarially with respect to $f'$ (which is the case in the game when $\p{f',x'}$ is chosen to be $\p{f^*,x^*}$)! This implies the following lemma, which shows that the posterior $Q_{\Bf'}$ based on the function $\Bf'$ that the defender receives does a good job of guessing the value of $f\p{\Bx'}$.

\begin{lemma}\label{lemma:b6}
    For the $\p{\Bf',\Bx'}$ received by the defender,
$$
\Pr\ss{\p{\Bf',\Bx'}\\\Bh \sim Q_{\Bf'}}\b{\Bh\p{\Bx'} \ne f\p{\Bx'}} \le \exp\p{2m\eps} \E_{\Bh \sim Q_f}\b{\abs{\Bh-f}}.
$$
\end{lemma}

\begin{proof}
    By Claim~\ref{claim:b5}, for any fixed $f,f' \in \CF$ with $\abs{f-f'} \le \eps$ and any input $x$,
\begin{align*}
\Pr_{\Bh \sim Q_{f'}}\b{\Bh(x) \ne f(x)} \le \exp\p{2m\eps}\Pr_{\Bh \sim Q_{f}}\b{\Bh(x) \ne f(x)},
\end{align*}
so in particular,
\begin{align*}
\Pr\ss{\p{\Bx',\Bf'}\\\Bh \sim Q_{\Bf'}}\b{\Bh\p{\Bx'} \ne f\p{\Bx'}}
&= \E_{\p{\Bx',\Bf'}}\b{\Pr_{\Bh \sim Q_{\Bf'}}\b{\Bh\p{\Bx'} \ne f\p{\Bx'}}}\\
&\le \exp\p{2m\eps} \E_{\p{\Bx',\Bf'}}\b{\Pr_{\Bh \sim Q_{f}}\b{\Bh\p{\Bx'} \ne f\p{\Bx'}}}\\
&= \exp\p{2m\eps} \E_{\Bx \sim \CD}\b{\Pr_{\Bh \sim Q_{f}}\b{\Bh\p{\Bx} \ne f\p{\Bx}}}\\
&= \exp\p{2m\eps} \E_{\Bh \sim Q_{f}}\b{\Pr_{\Bx \sim \CD}\b{\Bh\p{\Bx} \ne f\p{\Bx}}}\\
&= \exp\p{2m\eps} \E_{\Bh \sim Q_{f}}\b{\abs{\Bh-f}}.\qedhere
\end{align*}
\end{proof}

Note that the above proof didn't care about whether $\Bf'$ happened to be $f$ or $\Bf^*$; it only used the fact that $\abs{\Bf'-f} \le \eps$. Also, as we already hinted above, the expression $\exp\p{2m\eps} \E_{\Bh \sim Q_f}\b{\abs{\Bh-f}}$ suggests we'll need to pick a value for $m$ which trades off between
\begin{itemize}
\item the discrepancy $\exp(m\eps)$ between the $Q_f$ and $Q_{\Bf'}$, which gets worse as $m$ increases
\item and the error $\E_{\Bh \sim Q_f}\b{\abs{\Bh-f}}$ of the posterior centered at $f$, which gets better as $m$ increases.
\end{itemize}

\begin{corollary}\label{cor:b7}
    For any prior $P$ and inverse temperature $m$, the class $\CF$ is $\eps$-defendable with confidence $\ge 1- \exp\p{2m \eps} \max_f\E_{\Bh \sim Q_f}\b{\abs{\Bh-f}}$ using the Boltzmann strategy.
\end{corollary}

In Theorem~\ref{thm:b16}, we will see that for classes of VC-dimension $d$, an optimal choice of $m$ gives confidence $\ge 1-O\p{d\eps \log \f1{d\eps}}$ as long as the prior $P$ is chosen appropriately.

\subsubsection{Beyond worst-case}
Corollary~\ref{cor:b7} above gives a guarantee assuming that the Gibbs posterior $Q_f$ is broadly accurate \emph{no matter which function $f$ the adversary chooses}, which is a strong assumption. Here, we outline two ways to get bounds that can adapt to the particular choice of $f$ to sometimes provide stronger guarantees than those that can be obtained in the worst case.

First, let's study the case where $f$ is drawn at random from some fixed distribution $P'\in \tri\p\CF$ instead of being chosen adversarially. Let's say that $\CF$ is \emph{$\eps$-defendable on average over $P'$} with confidence $1-\delta$ if the defender wins the backdoor defense game with probability $\ge 1-\delta$ when $f$ is drawn from $P'$. Then we only need to show that $Q_\Bf$ is accurate on average for $\Bf \sim P'$:

\begin{corollary}\label{corr:b8}
    For any distribution $P'$, prior $P$, and inverse temperature $m$, the class $\CF$ is $\eps$-defendable on average over $P'$ with confidence $\ge 1- \exp(m \eps) \E\ss{\Bf \sim P'\\\Bh \sim Q_\Bf}\b{\abs{\Bh-\Bf}}$ using the Boltzmann strategy for $P$.
\end{corollary}

Second, even if we have no guarantee that $Q_f$ will be accurate in general, the defender can still check whether it is accurate in this particular instance of the game, simply by checking whether $Q_{\Bf'}$ is accurate. Indeed, the following lemma shows that the errors are within a small factor of each other as long as $m\eps$ is small.

\begin{lemma}\label{lemma:b9}
    For any $f,f' \in \CF$, we have
$$
\E_{\Bh \sim Q_f}\b{\abs{\Bh-f}} \le \exp\p{m\abs{f-f'}}\p{\E_{\Bh \sim Q_{f'}}\b{\abs{\Bh-f'}} + \abs{f-f'}}.
$$
\end{lemma}

\begin{proof}
    The distance $\abs{\Bh-f}$ is always nonnegative, so using Claim~\ref{claim:b5}, we can switch from $Q_f$ to $Q_{f'}$ and get
\begin{align*}
\E_{\Bh \sim Q_f}\b{\abs{\Bh-f}}
&\le \exp\p{m\abs{f-f'}}\E_{\Bh \sim Q_{f'}}\b{\abs{\Bh-f}}\\
&\le \exp\p{m\abs{f-f'}}\p{\E_{\Bh \sim Q_{f'}}\b{\abs{\Bh-f'}} + \abs{f-f'}}.\qedhere
\end{align*}
\end{proof}

\subsection{Conditions for accuracy}
In this section, we give bounds on the error of the posterior. Our results are closely related to the results of \cite*{haussler1994bounds}; but they give bounds on the error that an optimal learning algorithm gets after seeing $m$ examples, whereas we give bounds on the error that a Boltzmann posterior gets with an inverse temperature of $m$. These models come apart in the sense that in their setting, the optimal error rate for classes of VC-dimension $d$ is $\Th\p{\f m d}$ (due to a result in \cite*{hlw}), whereas in our setting the optimal error rate is $\Theta\p{\f m d \log \f d m}$. On an intuitive level, this is because the Boltzmann posterior never ``fully updates'': it is more akin to the posterior one might get after drawing $m$ examples that have a constant amount of classification noise, leaving some uncertainty.

\subsubsection{Average-case bounds for all priors}
First we prove some results which apply no matter what prior $P$ is chosen. The following lemma shows that the error on $f$ will be small as long as, for a typical sample of $m$ points $\Bx_1,\dots,\Bx_m$ from $\CD$, the prior $P$ places reasonably high probability on $f$'s labeling of $\Bx_1,\dots,\Bx_m$.

\begin{lemma}\label{lemma:b10}
    For any $f \in \CF$,
$$
\E_{\Bh \sim Q_f}\b{\abs{\Bh-f}} \le \f1m\E_{\Bx \sim \CD^m}\b{\log \f1{\Pr_{\Bh \sim P}\b{\forall i: \Bh\p{\Bx_i} = f\p{\Bx_i}}}}.
$$
\end{lemma}
\begin{proof}
    First, by Jensen's inequality applied to the concave function $g(t) \ce t \log \f1t$,
\begin{align*}
\E_{\Bh \sim P}\b{\exp\p{-m\abs{\Bh-f}}m\abs{h-f}} \le \E_{\Bh \sim P}\b{\exp\p{-m \abs{\Bh-f}}}\log \f1{\E_{\Bh \sim P}\b{\exp\p{-m \abs{\Bh-f}}}},
\end{align*}
so by rearranging, the error is at most
\begin{align}
\E_{\Bh \sim Q_f}\b{\abs{\Bh-f}}
&= \f{\E_{\Bh \sim P}\b{\exp\p{-m\abs{\Bh-f}}\abs{h-f}}}{\E_{\Bh \sim P}\b{\exp\p{-m \abs{\Bh-f}}}}\notag\\
&\le \f1m \log\f1{\E_{\Bh \sim P}\b{\exp\p{-m\abs{\Bh-f}}}},\label{eq:error-le-likelihood}
\end{align}
where $\E_{\Bh \sim P}\b{\exp\p{-m\abs{\Bh-f}}}$ is average ``Boltzmann likelihood'' of $f$ over the prior $P$. As we observed before, the likelihood $\exp\p{-m\abs{h-f}}$ is an upper bound on the probability that $h$ and $f$ agree on a sample of $m$ points, so we have
\begin{align*}
\log\f1{\E_{\Bh \sim P}\b{\exp\p{-m\abs{\Bh-f}}}}
&\le \log \f1{\Pr\ss{\Bh \sim P\\\Bx \sim \CD^m}\b{\forall i:\Bh(\Bx_i) = f(\Bx_i)}}\\
&\le\E_{\Bx \sim \CD^m}\b{\log \f1{\Pr_{\Bh \sim P}\b{\forall i:\Bh(\Bx_i) = f(\Bx_i)}}}.\qedhere
\end{align*}
\end{proof}

Notice that if we average this last line over $f$ drawn from $P$, it gives us the average entropy of the labeling of a random sample of $m$ points, which we'll call \emph{labeling entropy}:
\begin{align*}
\E_{\Bf \sim P}\b{\E_{\Bx \sim \CD^m}\b{\log \f1{\Pr_{\Bh \sim P}\b{\forall i:\Bh(\Bx_i) = \Bf(\Bx_i)}}}}
&= \E_{\Bx \sim \CD^m}\b{\E_{\Bf \sim P}\b{\log \f1{\Pr_{\Bh \sim P}\b{\forall i:\Bh(\Bx_i) = \Bf(\Bx_i)}}}}\\
&= \E_{\Bx \sim \CD^m}\b{\H_{\Bh \sim P}\bco{\Bh(\Bx_1),\dots,\Bh(\Bx_m)}{\Bx}}.
\end{align*}
From this, we immediately get the following average-case bound.

\begin{lemma}\label{lemma:b11}
    If $\CF$ has VC dimension $d$, then for any prior $P$ and inverse temperature $m \ge 2d$,
$$
\E_{\Bf \sim P} \b{\E_{\Bh \sim Q_\Bf}\b{\abs{\Bh-\Bf}}} \le \f d m \log O\p{\f m d}.
$$
\end{lemma}

\begin{proof}
    By the Sauer--Shelah lemma, a class of VC dimension $d$ can label a sample of $m \ge 2d$ points in at most $O\p{\f m d}^d$ ways, so by Lemma~\ref{lemma:b10} we have
\begin{align*}
\E_{\Bf \sim P} \b{\E_{\Bh \sim Q_\Bf}\b{\abs{\Bh-\Bf}}}
&\le \f1m\E_{\Bf \sim P}\b{\E_{\Bx \sim \CD^m}\b{\log \f1{\Pr_{\Bh \sim P}\b{\forall i:\Bh(\Bx_i) = \Bf(\Bx_i)}}}}\\
&= \f1m\E_{\Bx \sim \CD^m}\b{\H_{\Bh \sim P}\bco{\Bh(\Bx_1),\dots,\Bh(\Bx_m)}{\Bx}}\\
&\le \f1m\E_{\Bx \sim \CD^m}\b{\log\#\setco{\p{h(\Bx_1),\dots,h(\Bx_m)}}{h \in \CF}}\\
&\le \f1m\E_{\Bx \sim \CD^m}\b{\log\p{O\p{\f m d}^d}}\\
&= \f d m \log O\p{\f m d}.\qedhere
\end{align*}
\end{proof}

\begin{theorem}\label{thm:b12}
    For any class $\CF$ of VC dimension $d$, any distribution $P \in \tri(\CF)$, and any $\eps \le \f1{2d}$, $\CF$ is $\eps$-defendable on average over $P$ with confidence $\ge 1- O\p{d\eps \log \f1{d\eps}}$ using the Boltzmann strategy for $P$.
\end{theorem}

\begin{proof}
    By Lemmas~\ref{lemma:b11} and~\ref{lemma:b6},
\begin{align*}
\E_{\Bf \sim P}\b{\Pr\ss{\p{\Bx',\Bf'}\\\Bh \sim Q_{\Bf'}}\b{\Bh\p{\Bx'} \ne f\p{\Bx'}}}
&\le \E_{\Bf \sim P}\b{\exp\p{2m\eps} \E_{\Bh \sim Q_{f}}\b{\abs{\Bh-\Bf}}}\\
&\le \exp\p{2m\eps} \f d m \log O\p{\f m d},
\end{align*}
so we obtain the desired bound by setting $m \ce \f1\eps \ge 2d$.
\end{proof}

\subsubsection{Worst-case bounds for specific priors}
We now give two closely related priors which give the same $O\p{d\eps \log \f1{d\eps}}$ bound for classes of VC-dimension $d$, but this time the guarantee holds for a worst-case $f$, instead of holding only on average over $f$ drawn from the prior. This is quite a surprising result: it shows that the posterior $Q_{f'}$ is still accurate despite the fact that the prior it's based on is in some sense ``incorrect''.

\begin{definition}\label{def:b13}
    Let $\pme$ be any distribution which maximizes the $m$-sample labeling entropy (the entropy of the labelings produced by hypotheses drawn from the prior). That is, let $\pme$ be any distribution such that
    \begin{align*}\E_{\Bx \sim \CD^m}\b{\H_{\Bh \sim \pme}\bco{\Bh(\Bx_1),\dots,\Bh(\Bx_m)}{\Bx}} = \max_P \E_{\Bx \sim \CD^m}\b{\H_{\Bh \sim P}\bco{\Bh(\Bx_1),\dots,\Bh(\Bx_m)}{\Bx}}.\end{align*}
\end{definition}

\begin{definition}\label{def:b14}
    Let $\pc$ be any distribution formed by a random process which takes the following form:
\begin{itemize}
\item draw some sample $\Bx_1',\dots,\Bx_m' \sim \CD$ from the input distribution,
\item pick a \emph{uniformly random} labeling $\Bl \in \zo^m$ within $\setco{\p{h(\Bx_1'),\dots,h(\Bx_m')}}{h \in \CF}$,
\item pick an arbitrary hypothesis $h \in \CF$ that is consistent with the labeling $\Bl$.
\end{itemize}
\end{definition}

\begin{lemma}\label{lemma:b15}
    If $\CF$ has VC-dimension $d$, and the prior $P$ is either $\pme$ or $\pc$, then for any inverse temperature $m \ge 2d$,
$$
\E_{\Bh \sim Q_f}\b{\abs{\Bh-f}} \le O\p{\f d m \log \f m d}.
$$
\end{lemma}
\begin{proof}
    Let's start with $\pme$. By Lemma~\ref{lemma:b10} it's enough to show that
$$
\E_{\Bx \sim \CD^m}\b{\log \f1{\Pr_{\Bh \sim \pme}\b{\forall i:\Bh(\Bx_i) = f(\Bx_i)}}} \le O\p{d \log \f m d}.
$$
Since $\pme$ maximizes the entropy $\E_{\Bx \sim \CD^m}\b{\H_{\Bh \sim P}\bco{\Bh(\Bx_1),\dots,\Bh(\Bx_m)}{\Bx}}$, we must also have that for every $f \in \CF$,
$$
\E_{\Bx \sim \CD^m}\b{\log \f1{\Pr_{\Bh \sim\pme}\b{\forall i:\Bh(\Bx_i) = f(\Bx_i)}}} = \E_{\Bx \sim \CD^m}\b{\H_{\Bh \sim \pme}\bco{\Bh(\Bx_1),\dots,\Bh(\Bx_m)}{\Bx}}
$$
(this is easy to check by the method of Lagrange multipliers). We can then conclude by the fact that $\CF$ can label $m$ points in at most $O\p{\f m d}^d$ ways.

Let's move to $\pc$. By Equation~\eqref{eq:error-le-likelihood}, it's enough to show that
$$
\E_{\Bh \sim \pc}\b{\exp\p{-m\abs{\Bh-f}}} \ge \p{\f d m}^{O(d)}
$$
It is a well-known fact from VC theory that with probability $99\%$ over a random sample of $m$ points from $\CD$, all hypotheses $h$ consistent with $f$'s labeling have $\abs{h - f} \le O\p{\f d m \log \f m d}$. Suppose this happens during the random process in Definition~\ref{def:b14}, then given there are only $O\p{\f m d}^d$ labelings to choose from, $\Bl$ will be $f$'s labeling with probability $\ge O\p{\f d m}^d$, in which case $h$ will satisfy $\abs{h-f} \le O\p{\f d m \log \f m d}$. This means that
$$
\Pr_{\Bh \sim \pc}\b{\abs{\Bh - f} \le O\p{\f d m \log \f m d}} \ge 99\% \times O\p{\f d m}^d \ge O\p{\f d m}^d,
$$
so
\begin{align*}
\E_{\Bh \sim \pc}\b{\exp\p{-m\abs{\Bh-f}}}
&\ge \Pr_{\Bh \sim \pc}\b{\abs{\Bh - f} \le O\p{\f d m \log \f m d}} \exp\p{-m\times O\p{\f d m\log \f m d}}\\
&\ge \p{\f d m}^{O(d)}
\end{align*}
as desired.
\end{proof}

\begin{theorem}\label{thm:b16}
    For any $\eps \le \f1{2d}$, $\CF$ is $\eps$-defendable on average over $P$ with confidence $\ge 1- O\p{d\eps \log \f1{d\eps}}$ using the Boltzmann strategy given by either $\pme$ or $\pc$.
\end{theorem}

\begin{proof}
    By Lemmas~\ref{lemma:b15} and~\ref{lemma:b6}, for any $f \in \CF$,
\begin{align*}
\Pr\ss{\p{\Bx',\Bf'}\\\Bh \sim Q_{\Bf'}}\b{\Bh\p{\Bx'} \ne f\p{\Bx'}}
&\le \exp\p{2m\eps} \E_{\Bh \sim Q_{f}}\b{\abs{\Bh-f}}\\
&\le \exp\p{2m\eps} O\p{\f d m \log \f m d},
\end{align*}
so we obtain the desired bound by setting $m \ce \f1\eps \ge 2d$.
\end{proof}

\subsubsection{Matching lower bound}
The following result shows that the bound $O\p{\f d m \log \f m d}$ we obtained in Lemmas~\ref{lemma:b15} and~\ref{lemma:b11} is tight, which means that the best bound on $\delta$ that can be obtained by any approach based on Lemma~\ref{lemma:b6} is $\Th\p{d \eps \log \f1{d\eps}}$.

\begin{theorem}\label{thm:b17}
    For every integer $d$ and inverse temperature $m\ge 2d$, there is a class $\CF$ of VC dimension $d$ such that
\begin{itemize}
\item for any prior $P$, there is a function $f$ such that $\E_{\Bh \sim Q_f}\b{\abs{\Bh-f}} \ge \Om\p{\f d m \log \f m d}$,
\item there is a prior $P$ such that $\E_{\Bf \sim P}\b{\E_{\Bh \sim Q_\Bf}\b{\abs{\Bh-\Bf}}} \ge \Om\p{\f d m \log \f m d}$.
\end{itemize}
\end{theorem}

\begin{proof}
    For any $d$ and desired error rate $\delta > 0$, the proof of Theorem 6 in \cite{gentile1998improved} gives a class $\CF$ of VC dimension $\le d$ and a distribution $\CD$ such that $\CF$ contains $\p{\f1\delta}^{\Th(d)}$ functions and the distance between any two distinct functions $h,f \in \CF$ satisfies $2\delta \le \abs{h-f} \le O\p{\delta}$.

The lower bound $\abs{h-f} \ge 2\delta$ means that in order to get $\E_{\Bh \sim Q_f}\b{\abs{\Bh - f}} \le \delta$ for some function $f$, at least half of the probability mass of $Q_f$ must be on $f$. On the other hand, the mass of $f$ in $Q_f$ is given by
$$
Q_f\p{f} = \f{P(f)\exp\p{-m\abs{f-f}}}{\E_{\Bh \sim P}\b{\exp\p{-m\abs{\Bh-f}}}} \le \f{P(f)\exp\p{-m\times0}}{\exp\p{-m\times O\p\delta}} = P(f)\exp\p{O\p{\delta m}},
$$
so we must have
$$
P(f)\exp\p{O\p{\delta m}} \ge \f12 \so \delta \ge \Om\p{\log \f{1}{2P(f)}}.
$$
Now, given that $\CF$ has $\p{\f1\delta}^{\Th(d)}$ functions,
\begin{itemize}
\item every prior $P$ must have a function $f$ with $P(f) \le \delta^{\Om(d)}$,
\item and the uniform prior $P$ gives every function $P(f) \le \delta^{\Om(d)}$,
\end{itemize}

so in either case we get
\begin{equation*}
\delta \ge \Om\p{\log \f{1}{2P(f)}} \ge \Om\p{\f d m \log \f1\delta} \so \delta \ge \Om\p{\f d m \log \f m d}.\qedhere
\end{equation*}
\end{proof}

\subsection{Comparison with the previous strategy}
In this section, we present some reasons why we think strategies like this one are more promising than the strategy based on re-learning presented in Section~\ref{section:unbounded} when it comes to catching real-life backdoors.

On a practical level, we have hope that this strategy or similar strategies will be less computationally costly than the strategy in Section~\ref{section:unbounded}:
\begin{itemize}
\item First, this strategy requires only \emph{one} draw from the posterior distribution $Q_{\Bf'}$, whereas the strategy in Section~\ref{section:unbounded} requires re-learning the function $\Th\p{\log \f1{d\eps}}$ times.
\item Second, learning an approximation of $Q_{\Bf'}$ (e.g. through variational inference) does not require re-learning the function from scratch. For example, one could generate reasonable posteriors by considering a Gaussian ball around the parameters of $\Bf'$, as is done by \citet{dziugaite2017computing} in the context of computing PAC-Bayes bounds.
\end{itemize}

On a more philosophical level, we think that approximations of this strategy (such as the ones just mentioned) may end up being more ``mechanistic'' in the sense discussed in Section~\ref{section:mechanistic}, and we are excited about developing the strategy in that direction. In addition, the fact that the posterior $Q_f$ is not very sensitive to small changes in $f$ suggests that it's taking a more neutral stance (from the perspective of the prior $P$). So there is some hope that it could detect ``backdoors'' in $f$ itself, or detect anomalous behaviors which are not adequately modeled as uniformly random small modifications of a reference ``clean'' function, which could be relevant to the motivation discussed in Section~\ref{section:alignment}.

\newpage

\section{Proof of Theorem \ref{theorem:unbounded}}\label{appendix:unbounded}

\restateunbounded*

\begin{proof}
Write $d=\operatorname{VC}\left(\mathcal F\right)$. We will show that, more specifically, for $d\geq 1$:
\begin{enumerate}[label=(\alph*)]
\item\label{unboundedattack}
For any detection strategy there is an adversary limiting their win probability to at most $\max\left(\frac 12,1-\frac 12\left(d-1\right)\varepsilon\right)$.
\item\label{unboundeddefense}
There is a detection strategy that wins against any adversary with probability at least $\max\left(\frac 12,1-15d\varepsilon\right)$.
\end{enumerate}

For part \ref{unboundedattack}, let $S$ be a set of size $d$ shattered by $\mathcal F$, and let $s\in S$ be any point. The adversary proceeds by taking $\mathcal D$ to assign probability $\min\left(\frac 1{d-1},\varepsilon\right)$ to each point of $S\setminus\left\{s\right\}$ and the remaining $1-p$ probability to the point $s$, where $p:=\min\left(1,\left(d-1\right)\varepsilon\right)$. The adversary then takes $f$ to be one of the $2^{d-1}$ witnesses for the shattering that assigns $0$ to the point $s$. We may treat $f$ as if it were chosen uniformly at random, since the detection strategy's worst-case performance can be no better than its average-case performance. Finally, given $x^\ast$, the adversary takes $f^\ast$ to be the witness that agrees with $f$ exactly on the set $S\setminus\left\{x^\ast\right\}$ if $x^\ast\neq s$, and otherwise takes $f^\ast=f$. It is easy to check that $f^\ast$ is $\varepsilon$-valid if and only if $x^\ast\neq s$. Furthermore, $\left(f,x\right)$ and $\left(f^\ast,x^\ast\right)$ are identically distributed for $x\sim\mathcal D$. Hence the best a detection strategy can do is to output $\textsc{Acc}$ when given the point $s$, and otherwise the strategy can do no better than chance. So the detection strategy's overall win probability is at most $\left(1-p\right)+\frac p2=1-\frac p2$, as required.

For part \ref{unboundeddefense}, note first that the detection strategy can achieve a win probability of at least $\frac 12$ simply by guessing randomly, so we may assume that $1-15d\varepsilon>\frac 12$. To construct the detection strategy, by Haussler--Littlestone--Warmuth, take a prediction strategy for $\mathcal F$ with sample size $m-1$ and error rate at most $\frac dm$, for some positive integer $m$ to be chosen later. The detection strategy then works as follows:
\begin{itemize}
\item
Given $\left(f^\prime,x^\prime\right)$ as input (which could be either $\left(f,x\right)$ or $\left(f^\ast,x^\ast\right)$), we use the prediction strategy to make a prediction $z$ for $f^\prime\left(x^\prime\right)$, which we in turn think of as a prediction for $f\left(x^\prime\right)$. The idea is that if we can successfully predict $f\left(x^\prime\right)$, then we can distinguish the two cases by comparing this to $f^\prime\left(x^\prime\right)$, since $f\left(x\right)=f\left(x\right)$ but $f\left(x^\ast\right)\neq f^\ast\left(x^\ast\right)$ if $f^\ast$ is $\varepsilon$-valid.
\item
We repeat this procedure $r$ times with independent samples to obtain predictions $z^{\left(1\right)},\dots,z^{\left(r\right)}$ for some positive integer $r$ to be chosen later.
\item
Finally, we take a majority vote: if more than half of $z^{\left(1\right)},\dots,z^{\left(r\right)}$ are different from $f^\prime\left(x^\prime\right)$, then we output $\textsc{Rej}$, and otherwise we output $\textsc{Acc}$.
\end{itemize}

To lower bound the detection strategy's win probability, consider first the case in which $\left(f,x\right)$ is passed to the detection strategy. In this case the probability that $z\neq f\left(x\right)$ is at most the error rate $\frac dm$. Hence by Markov's inequality and linearity of expectation, the probability that more than half of $z^{\left(1\right)},\dots,z^{\left(r\right)}$ are different from $f\left(x\right)$ is at most $\frac{2d}m$, giving the detection strategy a failure probability of at most $\frac{2d}m$. If instead $\left(f^\ast,x^\ast\right)$ is passed to the detection strategy, then there are two possible reasons why we could have $z\neq f\left(x^\ast\right)$: either $f$ and $f^\ast$ disagree on at least one of the $m-1$ samples provided by the example oracle, or the prediction strategy fails to predict $f\left(x^\ast\right)$ even if it is provided with the value of $f$ on all of these $m-1$ samples. Write $\mathbbm 1^{\left(i\right)}_{\text{$f$ and $f^\ast$ disagree}}$ and $\mathbbm 1^{\left(i\right)}_{\text{prediction fails}}$ for the indicator functions of these two events respectively during the $i$th trial for $i\in\left\{1,\dots,r\right\}$. Then the probability that the detection strategy fails is
\[\mathbb P\left(\frac 1r\sum_{i=1}^r\mathbbm 1_{z^{\left(i\right)}\neq f\left(x^\ast\right)}\geq\frac 12\right)\leq\mathbb P\left(\frac 1r\sum_{i=1}^r\mathbbm 1^{\left(i\right)}_{\text{$f$ and $f^\ast$ disagree}}\geq\frac 14\right)+\mathbb P\left(\frac 1r\sum_{i=1}^r\mathbbm 1^{\left(i\right)}_{\text{prediction fails}}\geq\frac 14\right).\]
The second term on the right-hand side is at most $\frac{4d}m$ by another application of Markov's inequality and linearity of expectation. To bound the first term on the right-hand side, we may assume without loss of generality that $f^\ast$ is $\varepsilon$-valid, which implies that $p:=\mathbb P\left(\mathbbm 1^{\left(i\right)}_{\text{$f$ and $f^\ast$ disagree}}=1\right)\leq\left(m-1\right)\varepsilon$, by the union bound. Taking $m$ to be the unique positive integer with $\frac 1{5\varepsilon}<m\leq\frac 1{5\varepsilon}+1$, we have $p\leq\frac 15$. Since these $r$ indicator functions are independent, the first term on the right-hand side is at most $\exp\left(-2\left(\frac 14-\frac 15\right)^2r\right)=\exp\left(-\frac r{200}\right)$ by Hoeffding's inequality, which can be made arbitrarily small by taking $r$ to be sufficiently large. Putting everything together, the detection strategy's best overall win probability is at least
\[\tfrac 12\left(\left(1-\tfrac{2d}m\right)+\left(1-\tfrac{4d}m\right)\right)=1-\tfrac{3d}m>1-15\varepsilon d,\]
as required.
\end{proof}

\newpage

\section{Proof of Theorem \ref{theorem:random}}\label{appendix:random}

\restaterandom*

\begin{proof}
Write $\textsc{Rand}:\left\{0,1\right\}^\ast\to\left\{0,1\right\}$ for the random oracle.
Let $\mathcal K=\left\{0,1\right\}^n$ be a set of ``keys'', and let $\mathcal F=\left\{x\mapsto\textsc{Rand}\left(K\parallel x\right)\;\middle|\;K\in\mathcal K\right\}$, where $\parallel$ denotes string concatenation. We will show that, over the choice of random oracle, $\mathcal F$ is efficiently defendable with probability $1-o(2^{-n^c})$ as $n\to\infty$ for all $c\geq 1$, but not efficiently PAC learnable with probability $1-o\left(\exp\left(-2^{n-4}\right)\right)$ as $n\to\infty$. The result then follows by the union bound.

For defendability, let $\delta>0$ and $c\geq 1$. We will show that, with probability $1-o(2^{-n^c})$ over the choice of random oracle, if $\varepsilon=o\left(\frac\delta{n^c}\right)$ then $\mathcal F$ is $\varepsilon$-defendable with confidence $1-\delta$ using the trivial detection strategy that always outputs $\textsc{Acc}$. To see this, let $\mathcal D$ be the distribution over $\mathcal X$ chosen by the adversary, and for $x\in\mathcal X$ write $\mathbb P_{\mathcal D}\left(x\right)$ for $\mathbb P_{x^\prime\sim\mathcal D}\left(x^\prime=x\right)$. In order for the adversary to be able to choose a backdoored function that is $\varepsilon$-valid with high enough probability over the choice of backdoor trigger, the adversary must choose $\mathcal D$ so that
\[S_1:=\sum_{\substack{x\in\mathcal X\\\mathbb P_{\mathcal D}\left(x\right)\leq\varepsilon}}\mathbb P_{\mathcal D}\left(x\right)>2\delta.\]
Note also that
\[S_2:=\sum_{\substack{x\in\mathcal X\\\mathbb P_{\mathcal D}\left(x\right)\leq\varepsilon}}\mathbb P_{\mathcal D}\left(x\right)^2\leq\varepsilon S_1.\]
Now for any $K,K^\ast\in\mathcal K$ and any $\varepsilon<\frac 12S_1$,
\begin{align*}
&\mathbb P_{\textsc{Rand}}\left(\mathbb P_{x\sim D}\left(\textsc{Rand}\left(K\parallel x\right)\neq\textsc{Rand}\left(K^\ast\parallel x\right)\right)\leq\varepsilon\right)\\
&\leq\mathbb P_{\textsc{Rand}}\left(\sum_{\substack{x\in\mathcal X\\\mathbb P_{\mathcal D}\left(x\right)\leq\varepsilon}}\mathbb P_{\mathcal D}\left(x\right)\mathbbm 1_{\textsc{Rand}\left(K\parallel x\right)\neq\textsc{Rand}\left(K^\ast\parallel x\right)}\leq\varepsilon\right)\\
&\leq\exp\left(-\frac{2\left(\frac 12S_1-\varepsilon\right)^2}{S_2}\right)&&\text{(by Hoeffding's inequality)}\\
&\leq\exp\left(-\frac{S_1}{2\varepsilon}-\frac{2\varepsilon}{S_1}+2\right)&&\text{(since $S_2\leq\varepsilon S_1$)}\\
&<\exp\left(-\frac\delta\varepsilon-2\varepsilon+2\right)&&\text{(since $2\delta<S_1\leq 1$).}
\end{align*}
Hence by the union bound, it is impossible for the adversary to choose any $K,K^\ast\in\mathcal K$ with $\mathbb P_{x\sim D}\left(\textsc{Rand}\left(K\parallel x\right)\neq\textsc{Rand}\left(K^\ast\parallel x\right)\right)\leq\varepsilon$ with probability at least
\[1-\binom{\left|\mathcal K\right|}{2}\exp\left(-\frac\delta\varepsilon-2\varepsilon+2\right)>1-\exp\left(\left(2n-1\right)\log\left(2\right)-\frac\delta\varepsilon-2\varepsilon+2\right).\]
When this event happens the detection strategy always wins, and if  $\varepsilon=o\left(\frac\delta {n^c}\right)$ then this probability is $1-o(2^{-n^c})$, as required.

For PAC learnability, we will show that, with probability $1-o\left(\exp\left(-2^{n-4}\right)\right)$ over the choice of random oracle, there is no PAC learning algorithm for $\mathcal F$, even with access to the random oracle, that runs in time $f\left(n\right)=1.5^n$ with confidence parameter $\frac 14$ and error parameter $\frac 14$, say. To see this, note that, for all sufficiently large $n$, there are fewer than $2^{f\left(n\right)}$ possible hypotheses that such an algorithm could output, assuming by convention that algorithms can process at most one bit per unit time. Now for each of these hypotheses $h$, if $\mathcal D$ is the uniform distribution over $\mathcal X$ and $K\in\mathcal K$, then
\[\mathbb P_{\textsc{Rand}}\left(\mathbb P_{x\sim D}\left(\textsc{Rand}\left(K\parallel x\right)\neq h\left(x\right)\right)\leq\tfrac 14\right)\leq\exp\left(-2\left(\tfrac 12-\tfrac 14\right)^2\left|\mathcal X\right|\right)=\exp\left(-2^{n-3}\right)\]
by Hoeffding's inequality. Hence by the union bound, the probability that none of these hypotheses has generalization error at most $\frac 14$ is at least
\[1-2^{f\left(n\right)}\exp\left(-2^{n-3}\right)=1-\exp\left(1.5^n\log\left(2\right)-2^{n-3}\right)=1-o\left(\exp\left(-2^{n-4}\right)\right).\]
When this event happens the PAC learning algorithm fails with probability $1>\frac 14$, as required.
\end{proof}

\newpage

\section{Proof of Theorem \ref{theorem:trees}}\label{appendix:trees}

\restatetrees*

\begin{proof}
Given $f\in\mathcal F$ and $x\in\mathcal X$, write $\textsc{Depth}\left(f,x\right)$ for the number of distinct input variables that appear along the path from root to leaf when $f$ is evaluated at $x$.

Let $\delta>0$, and take $\varepsilon<\frac{\delta^2}{s^2}$. We claim that, as long as the adversary takes $\mathcal D$ to be uniform distribution, $\mathcal F$ is $\varepsilon$-defendable with confidence $1-\delta$ using the detection strategy that, given $\left(f^\prime,x^\prime\right)$ as input, outputs $\textsc{Rej}$ if $2^{-\textsc{Depth}\left(f^\prime,x^\prime\right)}\leq\frac\delta s$ and $\textsc{Acc}$ otherwise. This suffices, since evaluating $\textsc{Depth}\left(f^\prime,x^\prime\right)$ takes the same time as evaluating $f^\prime\left(x^\prime\right)$, up to a constant factor.

To see this, given $f\in\mathcal F$ and $x\in\mathcal X$, write
\[\textsc{Leaf}\left(f,x\right)=\left\{x^\prime\in\mathcal X\;\middle|\;\text{$x^\prime$ and $x$ reach the same leaf of $f$}\right\},\]
and observe that
\[\mathbb P_{\mathcal D}\left(\textsc{Leaf}\left(f,x\right)\right)=2^{-\textsc{Depth}\left(f,x\right)},\]
where we have written $\mathbb P_{\mathcal D}\left(S\right)$ as shorthand for $\mathbb P_{x\sim\mathcal D}\left(x\in S\right)$.
The key claim is that, for any $f\in\mathcal F$ and any $x^\ast\in\mathcal X$, if $f^\ast\in\mathcal F$ is $\varepsilon$-valid, then
\[2^{-\textsc{Depth}\left(f,x^\ast\right)-\textsc{Depth}\left(f^\ast,x^\ast\right)}\leq\mathbb P_{\mathcal D}\left(\textsc{Leaf}\left(f,x^\ast\right)\cap\textsc{Leaf}\left(f^\ast,x^\ast\right)\right)\leq\varepsilon.\]
The first inequality is a consequence of the fact that the number of distinct variables that appear when either $f$ or $f^\ast$ is evaluated at $x^\ast$ is at most $\textsc{Depth}\left(f,x^\ast\right)+\textsc{Depth}\left(f^\ast,x^\ast\right)$, and makes essential use of the fact that $\mathcal D$ is uniform (or at the very least, a product distribution). The second inequality follows from the definition of $\varepsilon$-valid.

Now choose $f$, $x^\ast$, $f^\ast$ and $x$ as in the backdoor detection game. Suppose first that $\left(f,x\right)$ is passed to the detection strategy. Then with probability at least $1-\delta$, we have $2^{-\textsc{Depth}\left(f,x\right)}=\mathbb P_{\mathcal D}\left(\textsc{Leaf}\left(f,x\right)\right)>\frac\delta s$, since $f$ has at most $s$ leaves. Hence the detection strategy outputs $\textsc{Acc}$ with probability at least $1-\delta$. Now suppose instead that $\left(f^\ast,x^\ast\right)$ is passed to the detection strategy. Then with probability at least $1-\delta$, we have $2^{-\textsc{Depth}\left(f,x^\ast\right)}>\frac\delta s$, by a similar argument. But if $f^\ast$ is $\varepsilon$-valid, then this implies that $2^{-\textsc{Depth}\left(f^\ast,x^\ast\right)}<\nicefrac\varepsilon{\frac\delta s}<\frac\delta s$, by the key claim. Hence the detection strategy outputs $\textsc{Rej}$ with probability at least $1-\delta$. So the detection strategy's overall win probability is at least $1-\delta$, as required.
\end{proof}

\end{document}